\documentclass{article}

\pdfoutput=1

\usepackage{microtype}
\usepackage{graphicx}
\usepackage{subfigure}
\usepackage{booktabs} 

\usepackage{hyperref}

\usepackage[accepted]{affol}

\usepackage{amsmath}
\usepackage{amssymb}
\usepackage{mathtools}
\usepackage{amsthm}
\usepackage{pgfplots}
\usepackage{comment}
\usepgfplotslibrary{colorbrewer}
\usepgfplotslibrary{fillbetween}
\usepgfplotslibrary{groupplots}
\usepackage{makecell}

\usepackage[capitalize,noabbrev]{cleveref}

\definecolor{mypink1}{rgb}{0.858, 0.188, 0.478}
\definecolor{mypink2}{RGB}{219, 48, 122}
\definecolor{mypink3}{cmyk}{0, 0.7808, 0.4429, 0.1412}
\definecolor{mygray}{gray}{0.6}

\definecolor{mycolor8}{rgb}{0, 0, 1}
\definecolor{mycolor7}{rgb}{0.15, 0.15, 0.9}
\definecolor{mycolor6}{rgb}{0.3, 0.3, 0.75}
\definecolor{mycolor5}{rgb}{0.45, 0.45, 0.6}
\definecolor{mycolor4}{rgb}{0.6, 0.6, 0.45}
\definecolor{mycolor3}{rgb}{0.75, 0.75, 0.3}
\definecolor{mycolor2}{rgb}{0.9, 0.9, 0.15}
\definecolor{mycolor1}{rgb}{1, 1, 0}

\definecolor{mycolor8_}{rgb}{1, 0, 0}
\definecolor{mycolor7_}{rgb}{0.9, 0.15, 0.15}
\definecolor{mycolor6_}{rgb}{0.75, 0.3, 0.3}
\definecolor{mycolor5_}{rgb}{0.6, 0.45, 0.45}
\definecolor{mycolor4_}{rgb}{0.45, 0.6, 0.6}
\definecolor{mycolor3_}{rgb}{0.3, 0.75, 0.75}
\definecolor{mycolor2_}{rgb}{0.15, 0.9, 0.9}
\definecolor{mycolor1_}{rgb}{0, 1, 1}

\definecolor{mycolor8__}{rgb}{0, 1, 0}
\definecolor{mycolor7__}{rgb}{0.15, 0.9, 0.15}
\definecolor{mycolor6__}{rgb}{0.3, 0.75, 0.3}
\definecolor{mycolor5__}{rgb}{0.45, 0.6, 0.45}
\definecolor{mycolor4__}{rgb}{0.6, 0.45, 0.6}
\definecolor{mycolor3__}{rgb}{0.75, 0.3, 0.75}
\definecolor{mycolor2__}{rgb}{0.9, 0.15, 0.9}
\definecolor{mycolor1__}{rgb}{1, 0, 1}

\theoremstyle{plain}
\newtheorem{theorem}{Theorem}[section]
\newtheorem{proposition}[theorem]{Proposition}
\newtheorem{lemma}[theorem]{Lemma}
\newtheorem{corollary}[theorem]{Corollary}
\theoremstyle{definition}
\newtheorem{definition}[theorem]{Definition}
\newtheorem{assumption}[theorem]{Assumption}
\theoremstyle{remark}
\newtheorem{remark}[theorem]{Remark}
\newtheorem{example}[theorem]{Example}

\newcommand{\R}{\mathbb{R}}

\newcommand{\N}{\mathbb{N}}
\newcommand{\E}{\mathbb{E}}

\newcommand{\Ell}{\mathcal{L}}

\newcommand{\Prob}{\mathcal{P}}
\newcommand{\Id}{\mathcal{I}}

\DeclareMathOperator{\support}{supp}

\DeclareMathOperator*{\argmin}{arg\,min}
\DeclareMathOperator{\erf}{erf}

\usepackage[textsize=tiny]{todonotes}

\affoltitlerunning{A Framework for Overparameterized Learning}

\begin{document}

\twocolumn[
\affoltitle{A Framework for Overparameterized Learning}



\affolsetsymbol{equal}{*}

\begin{affolauthorlist}
\affolauthor{D\'avid Terj\'ek}{rrr}
\affolauthor{Diego Gonz\'alez-S\'anchez}{rrr}
\end{affolauthorlist}

\affolaffiliation{rrr}{Alfr\'ed R\'enyi Institute of Mathematics, Budapest, Hungary}

\affolcorrespondingauthor{D\'avid Terj\'ek}{dterjek@renyi.hu}
\affolcorrespondingauthor{Diego Gonz\'alez-S\'anchez}{diegogs@renyi.hu}

\affolkeywords{}

\vskip 0.3in
]



\printAffiliationsAndNotice{}  






\begin{abstract}
A candidate explanation of the good empirical performance of deep neural networks is the implicit regularization effect of first order optimization methods. Inspired by this, we prove a convergence theorem for nonconvex composite optimization, and apply it to a general learning problem covering many machine learning applications, including supervised learning. We then present a deep multilayer perceptron model and prove that, when sufficiently wide, it $(i)$ leads to the convergence of gradient descent to a global optimum with a linear rate, $(ii)$ benefits from the implicit regularization effect of gradient descent, $(iii)$ is subject to novel bounds on the generalization error, $(iv)$ exhibits the lazy training phenomenon and $(v)$ enjoys learning rate transfer across different widths. The corresponding coefficients, such as the convergence rate, improve as width is further increased, and depend on the even order moments of the data generating distribution up to an order depending on the number of layers. The only non-mild assumption we make is the concentration of the smallest eigenvalue of the neural tangent kernel at initialization away from zero, which has been shown to hold for a number of less general models in contemporary works. We present empirical evidence supporting this assumption as well as our theoretical claims.
\end{abstract}

\section{Introduction} \label{introduction}
Explaining the success of highly overparameterized models such as deep neural networks is a central problem in the theory of modern machine learning \citep{Belkin2021}. Classical theory would imply that such models are prone to overfit to the training data. On the contrary, practice shows that while this can happen at moderate overparameterization, around the interpolation threshold where a model is just expressive enough to perfectly fit the data, further increasing model capacity leads to better generalization performance. This so-called double descent phenomenon \citep{Belkinetal2019} is often attributed to the implicit regularization effect of gradient descent and its variants, which are widely used for the training of deep neural networks. Some theoretical works \citep{Oymaketal2019a, Liuetal2022} have been proposing that optimization problems in modern machine learning enjoy the so-called Polyak-{\L}ojasiewicz (PL) condition \citep{Polyak1963}. Together with the Lipschitz gradient (LG) condition, they imply linear convergence of gradient descent to a global optimum. Moreover, by implicit regularization, the particular optimum the algorithm converges to is one that is close to the initial point, a property which could potentially explain why models trained in such a manner enjoy excellent generalization performance. The PL condition is closely linked to the smallest eigenvalue of the neural tangent kernel (NTK) \citep{Jacotetal2018, Liuetal2022}. The concentration of the smallest eigenvalue of the NTK at initialization has been the subject of many works \citep{Montanarietal2020, Nguyenetal2021, Wangetal2021, Bombarietal2022}. The so-called lazy training phenomenon \citep{Chizatetal2019}, which is when a model behaves similarly to its linearization, is responsible for the smallest eigenvalue to stay positive along the optimization path, which has been shown to lead to global convergence.

In this work, we propose a prototype optimization problem covering a range of machine learning applications including supervised learning, formulating them as instances of nonconvex composite optimization problems. We then prove a convergence theorem for nonconvex composite optimization generalizing the classical work of \citet{Polyak1963}, inspired by the more recent results of \citet{Oymaketal2019a} and \citet{Liuetal2022}. Applying it to the prototype problem we naturally arrive at conditions concerning the smallest eigenvalue of the NTK, as well as bounds on the network Jacobian and the Lipschitz constant of the Jacobian mapping. We then propose a general multilayer perceptron (MLP) model and show that, if sufficiently wide, then with high probability with respect to sampling both the initial parameters (from the prior) and the dataset (from the data generating distribution), the conditions of our theorem hold, and therefore we have convergence at a linear rate to a global optimum that is close to initialization. Our MLP formulation enables us to derive a bound on the Lipschitz constant of the trained model, leading to bounds on the generalization error. The only non-mild assumption that we make is that for sufficient width, the smallest eigenvalue of the NTK at initialization is bounded away from $0$ with high probability, which has been shown to hold for less general MLPs \citep{Montanarietal2020, Nguyenetal2021, Wangetal2021, Bombarietal2022}. Our MLP formulation leads naturally to lazy training, which allows us to only require the concentration of the smallest eigenvalue of the NTK at initialization, since it does not change too much during training. Moreover, as width is further increased, the probability of these events increases as well, the convergence ratio improves, the implicit regularization effect gets stronger and training ``gets lazier'' (leading to the NTK staying constant during training in the infinitely wide limit). A novel insight is that the corresponding coefficients depend on the even order moments of the data generating distribution, up to order $2(J-1)$ for $J$ layers (or $J-1$ hidden layers). An additional benefit of our MLP formulation is that it allows learning rate transfer across models of different width, similarly to the Maximal Update Paremeterization of \citet{Yangetal2022}.

After concluding Section~\ref{introduction} with listing our contributions in Subsection~\ref{contributions}, we introduce some notation and definitions in Section~\ref{preliminaries}. In Section~\ref{composite} we propose our convergence theorem for nonconvex composite optimization problems. Then in Section~\ref{main}, we present the prototype problem in Subsection~\ref{prototype} along with requirements on its components needed to ensure that the conditions of our convergence theorem hold, the MLP model satisfying the requirements in Subsection~\ref{mlp}, our convergence theorem for overparameterized learning in Subsection~\ref{convergence} and our theorem on generalization bounds in Subsection~\ref{generalization}. We follow by experiments supporting our theory in Section~\ref{experiments}, then review related work in Section~\ref{related} and discuss the limitations of our work in Section~\ref{limitations} along with future directions. Throughout the paper, we refer to the appendices for rigorous proofs of our results, examples of machine learning applications covered by the prototype problem, and experimental details.

\subsection{Contributions} \label{contributions}
\begin{itemize}
\item A convergence theorem of independent interest for nonconvex composite optimization.
\item A prototype optimization problem covering many machine learning applications with requirements for the components that ensure global convergence of gradient descent.
\item A deep multilayer perceptron model that, when sufficiently wide, satisfies the requirements with high probability, with the only non-mild assumption being that the smallest eigenvalue of the NTK is bounded away from $0$ at initialization with high probability.
\item The corresponding convergence theorem covering many machine learning applications (including supervised learning with losses satisfying the LG and PL conditions), exhibiting both global convergence with a linear rate, implicit regularization, lazy training and learning rate transfer.
\item A theorem bounding the generalization error of the trained model.
\item Experimental results supporting the NTK assumption and our theoretical results.
\end{itemize}

\section{Preliminaries}\label{preliminaries}

In this paper, $G$ and $H$ will always denote Hilbert spaces. Given $x \in G$ and $R>0$ we denote by $B(x,R)$ (resp., $\overline{B}(x,R)$) the open (resp., closed) ball with radius $R$ centered at $x$. The space of bounded linear operators from $G$ to $H$ is denoted $\Ell(G,H)$, and we equip it with the operator norm. The Frobenius norm of matrices is denoted $\Vert \cdot \Vert_F$, while the infinity and Lipschitz norms of functions are denoted $\Vert \cdot \Vert_\infty$ and $\Vert \cdot \Vert_L$, respectively. For a finite dimensional linear operator $A$ we denote its smallest and largest eigenvalues by $\lambda_{\min}(A)$ and $\lambda_{\max}(A)$, respectively. The adjoint of a linear operator $A \in \Ell(G,H)$ is the unique linear operator $A^* \in \Ell(H,G)$ such that $\langle Ax, y \rangle = \langle x, A^*y \rangle$ for all $x \in G$ and $y \in H$. Given a function $F:G\to H$ we say that it is differentiable if it is Fr\'echet differentiable, i.e., if there exists a bounded linear operator $\partial F(x) \in \Ell(G,H)$, which we refer to as the Jacobian of $F$ at $x$, satisfying $\lim_{y \to x }\frac{\Vert F(y) - F(x) - \partial F(x) (y - x) \Vert}{\Vert y - x \Vert}=0$. When $H=\R$ and $f:G\to \R$ is differentiable, we denote $\partial f (x) = \nabla f(x) \in G$ and refer to it as the gradient. Given any $f:G\to \R$ we denote its infimum by $f_*:=\inf_{x\in G}f(x)$.

\section{Nonconvex Composite Optimization} \label{composite}

The goal of this section is to prove sufficient conditions on a pair of functions $f:H\to \R$ and $F:G\to H$ in such a way that an infimum of $(f\circ F) : G \to \R$ can be found using gradient descent. We start by defining a number of conditions concerning $f$ and $F$.

\begin{definition}[Lipschitz Jacobian (LJ) and Lipschitz Gradient (LG) conditions]\label{def_lj_lg}
Let $G,H$ be Hilbert spaces, let $D \subset G$ and let $F : G \to H$ be a  differentiable function. Let $L \geq 0$ be a constant. We say that $F$ is $L$-LJ on $D$ if for all $x, y \in D$, one has $\Vert \partial F(x) - \partial F(y) \Vert \leq L \Vert x - y \Vert$. If $H=\R$ we denote this condition by Lipschitz Gradient (LG).
\end{definition}

\begin{definition}[Polyak-{\L}ojasiewicz (PL) condition]\label{def_pl}
Let $H$ be a Hilbert space, let $D \subset H$ and let $f : H \to \R$ be a differentiable function with $f_* \in \R$. Let $\lambda>0$ be a constant. We say that $f$ is $\lambda$-PL on $D$ if for all $x \in D$, one has $\frac{1}{2}\Vert \nabla f(x) \Vert^2 \geq \lambda (f(x) - f_*)$.
\end{definition}

\begin{definition}[Bounded Jacobian (BJ) and Bounded Gradient (BG) conditions]\label{def_bj_bg}
Let $G,H$ be Hilbert spaces, let $D \subset G$ and let $F : G \to H$ be a  differentiable function. Let $K \geq 0$ be a constant. We say that $F$ is $K$-BJ on $D$ if for all $x \in D$, one has $\Vert \partial F(x) \Vert \leq K$. If $H=\R$ we denote this condition by Bounded Gradient (BG).
\end{definition}

For the last condition recall that given a linear operator $A \in \mathcal{L}(H,H)$ we say that it is $\lambda$-coercive for some $\lambda>0$ if for all $y\in H$ one has $\langle y,Ay\rangle \ge \lambda\|y\|^2$. If $H$ is finite dimensional, this is equivalent to $\lambda_{\min}(A) \geq \lambda$. The following is a generalization of \citet[Definition~3]{Liuetal2022}.

\begin{definition}[Uniform Conditioning (UC)]
Let $G,H$ be Hilbert spaces, let $D \subset H$ and let $F : G \to H$ be a differentiable function. We say that $F$ is $\lambda$-UC on $D$ for some $\lambda>0$ if $\partial F(x) {\partial F(x)}^* \in \Ell(H,H)$ is $\lambda$-coercive for all $x \in D$.
\end{definition}

We are now ready to state our theorem about the convergence of gradient descent on $(f \circ F)$. The proofs of all the results in this section can be found in Appendix~\ref{app:composite}.
\begin{theorem} \label{theorem_gd_f_circ_F}
Let $G,H$ be Hilbert spaces, let $D \subset G$ be bounded and let $x_0 \in D$. Let $K_F,L_F,\lambda_F,L_f,\lambda_f \ge 0$ be constants such that $\lambda_F \leq K_F^2$ and $\lambda_f \leq L_f$. And let $F : G \to H$ be $K_F$-BJ and $L_F$-LJ on $D$ and $f : H \to \R$ be $L_f$-LG and $\lambda_f$-PL on $H$. Let us define:
\begin{itemize}
    \item $K_f = \sqrt{2 L_f (f(F(x_0)) - f_*)}$,
    \item $K=K_F K_f$,
    \item $L = K_F^2 L_f + K_f L_F$,
    \item $\lambda = \lambda_F \lambda_f$ (so that $\lambda\le L$),
    \item $\alpha \in \left(0,\frac{2}{L}\right)$,
    \item $q = 1 + L \lambda \alpha^2 - 2 \lambda \alpha$ (so that $q \in (0,1)$) and
    \item $R=\frac{\alpha K}{1 - \sqrt{q}}$.
\end{itemize}
Define GD starting at $x_0$ as $x_i = x_{i-1} - \alpha \nabla (f \circ F)(x_{i-1})$ for $i \geq 1$. If $\overline{B}(x_0,R) \subset D$ and $F$ is $\lambda_F$-UC on $\overline{B}(x_0,R)$,
then for all $i \ge 0$
\[
(f \circ F) (x_i) - {f}_* 
\leq q^i ((f \circ F) (x_0) - {f}_*).
\]
And the sequence $x_i$ converges to some $x_* = \lim_{i \to \infty} x_i$ such that
\[
(f \circ F)(x_*) = {f}_*
\]
(so that in particular $\inf_{h\in H}f(h)=\inf_{x\in G}(f \circ F)(x)$) and $x_* \in \overline{B}(x_0,R)$ so that
\[
\Vert x_* - x_0 \Vert \leq R.
\]
Moreover, if $\hat{x}_* = \argmin_{x \in G : (f \circ F)(x) = f_*}{ \Vert x - x_0 \Vert }$, i.e., if $\hat{x}_* \in G$ is an optimum of $(f \circ F)$ that is closest to $x_0$, then one has
\[
\Vert x_* - x_0 \Vert 
\leq \frac{\alpha K_F^2 L_f \Vert \hat{x}_* - x_0 \Vert}{1 - \sqrt{q}}.
\]
\end{theorem}

This theorem generalizes the classical result of \citet{Polyak1963}, which corresponds to the case $G=H$ and $F$ being the identity mapping. The optimal learning rate is $\alpha=\frac{1}{L}$, giving the convergence rate $q=1-\frac{\lambda}{L}$. There are two main ideas. The first is that the BJ and LJ conditions on $F$ and the BG and LG conditions on $f$ together ensure that the composition $(f \circ F)$ is LG, and we can bound the BG constant of $f$ along the optimization trajectory. The second is that the UC condition on $F$ and the PL condition on $f$ lead to the composition almost satisfying the PL condition, with $f_*$ in place of ${(f \circ F)}_*$, but when this happens on a large enough set, the two infimums become equal, so that the composition satisfies exactly the PL condition.

Lazy training is when $F$ behaves similarly to its Taylor expansion \citep{Chizatetal2019}, which happens exactly if its LJ constant is small. In particular, a $0$-LJ function is affine. The following lemma can exploit lazy training to turn coercivity of $\partial F(x_0) {\partial F(x_0)}^*$ into UC on a ball.
\begin{lemma}\label{lemma_lazy_training}
Let $G,H$ be Hilbert spaces and let $x_0 \in G$. Let $K_F,L_F,\lambda_0,R \ge 0$ be constants such that $\lambda_0 \leq K_F^2$. And let $F : G \to H$ be $K_F$-BJ and $L_F$-LJ on $B(x_0,R)$ and let $\partial F(x_0) {\partial F(x_0)}^*$ be $\lambda_0$-coercive. Define $\lambda_F = \lambda_0 - 2 K_F L_F R$. One then has that if $\lambda_F>0$, then $F$ is $\lambda_F$-UC on $B(x_0,R)$.
\end{lemma}

Finally, we can bound the initial loss value $f(F(x_0))$ (to get a bound on $K_f$) via the following lemma (if we know that $F(x_0)$ is in some ball around the origin).
\begin{lemma}\label{lemma_initial_loss_value}
Let $H$ be a Hilbert space, $f : H \to \R$ an $L_f$-LG function and let $R>0$. One then has
\[
f(x) \leq (L_f R + \Vert \nabla f(0) \Vert) R + f(0)
\]
for any $x \in \overline{B}(0,R) \subset H$.
\end{lemma}

\section{Overparameterized Learning}\label{main}

\subsection{Prototype Problem}\label{prototype}

In order to define the prototype problem we need several definitions that we introduce in the sequel. One is the \emph{dataset} $\mu \in \Prob(X)$ represented by a probability measure on a Borel space $X$. Another is the \emph{neural network mapping} $N : X \times \Theta \to \R^l$ with $X$ being the \emph{input space}, the Hilbert space $\Theta$ being the \emph{parameter space}, $\R^l$ being the \emph{output space}, and $N(x,\theta)$ being measurable in $x$ and differentiable in $\theta$ for all $ (x,\theta) \in X \times \Theta$. Assuming that for all $ \theta \in \Theta$ the integral $\int \Vert N(x,\theta) \Vert^2 d\mu(x)$ exists and is finite, we consider the \emph{induced mapping} $N_\mu : \Theta \to L^2(\mu,\R^l)$ defined as $N_\mu(\theta)$ being the equivalence class of the function $N(\cdot,\theta) : X \to \R^l$ with respect to $\mu$ for any $\theta \in \Theta$. The Hilbert space $L^2(\mu,\R^l)$ of equivalence classes of square integrable $\R^l$-valued functions with respect to $\mu$ is the \emph{feature space} with norm $\Vert f \Vert = \sqrt{\int \Vert f(x) \Vert^2 d\mu(x)}$ for $f \in L^2(\mu,\R^l)$. The last component is the \emph{integrand} $\iota : X \times \R^l \to \R$, mapping an input $x \in X$ and an output $z \in \R^l$ to a loss value $\iota(x,z) \in \R$, with $\iota(x,z)$ being measurable in $x$ and differentiable in $z$ for all $(x,z) \in X \times \R^l$. We assume that for all $x \in X$, $\iota(x,\cdot) : \R^l \to \R$ satisfies the LG and PL conditions globally with constants $L_\Ell$ and $\lambda_\Ell$, respectively, $\iota(x,\cdot)_*=\iota_* \in \R$ for all $x \in X$, and that the induced integral functional \citep{Rockafellar1976} $\Ell_\mu : L^2(\mu,\R^l) \to \R$, called the \emph{loss functional} and defined as $\Ell_\mu(f) = \int \iota(x,f(x)) d\mu(x)$ for $f \in L^2(\mu,\R^l)$, is finite for all $f \in L^2(\mu,\R^l)$.

\begin{definition}[Prototype problem]
With the assumptions above, the \emph{prototype problem} consists on finding an optimal parameter $\theta_*$ that attains
\[
\min_{\theta \in \Theta}{ (\Ell_\mu \circ N_\mu)(\theta) }.
\]
\end{definition}

This problem formulation enables us to translate not only supervised learning problems in general, but variational autoencoders, and even gradient regularized discriminators for generative adversarial networks to particular instances of the prototype problem as shown in Appendix~\ref{app:prototype}. The lemma below shows that the LG and PL properties of the integrand are inherited by the loss functional. Proofs of the results in this subsection can be found in Appendix~\ref{app:prototype}.

\begin{lemma}\label{lem:cond-iota}
With the above assumptions and for all $f \in L^2(\mu,\R^l)$, one has that
\[
\nabla \Ell_\mu(f)(x) = \nabla_z \iota(x,f(x))
\]
for $\mu$-a.e. $x \in X$ and $\Ell_\mu$ is $L_\Ell$-LG and $\lambda_\Ell$-PL globally.
\end{lemma}

Supervised learning problems with LG and PL losses can be defined by such integrands. Let $Y$ be a target Borel space, $t : X \to Y$ a measurable target function and $\ell : Y \times \R^l \to \R$ a loss function. Then $\iota : X \times \R^l \to \R$ defined as $\iota(x,z) = \ell(t(x),z)$ is measurable in $x$ and differentiable in $z$ for all $(x,z) \in X \times \R^l$, while if $\ell(y,\cdot)$ is $L_\Ell$-LG and $\lambda_\Ell$-PL for all $y \in Y$ with ${\ell(y,\cdot)}_* = \iota_*$ then clearly $\iota(x,\cdot)$ is $L_\Ell$-LG and $\lambda_\Ell$-PL for all $x \in X$ with ${\iota(x,\cdot)}_* = \iota_*$. Together with Lemma~\ref{lem:cond-iota}, this shows that supervised learning loss functionals defined as $\Ell_\mu(f) = \int \ell(f(x),t(x)) d\mu(x)$ inherit the LG and PL conditions from $\ell$. Examples of such losses are the least squares loss $\ell(z,y) = \frac{1}{2}\Vert z - y \Vert^2$ for $z \in \R^l$ and $y \in Y=\R^l$ with $t \in L^2(\mu,\R^l)$, which is $1$-LG and $1$-PL, as well as the regularized classification loss $\ell(z,y) = \log \sum e^z - z_y + \frac{\lambda}{2}\Vert z \Vert^2$ for $\lambda>0$, $z \in \R^l$ and $y \in [1:l]$, which is $(1+\lambda)$-LG and $\lambda$-PL.

In order to apply Theorem~\ref{theorem_gd_f_circ_F} to the prototype problem, some assumptions on the induced mapping $N_\mu$ need to hold. The lemma below provides sufficient conditions for the BJ and LJ properties of the neural network mapping $N$ for the induced mapping $N_\mu$ to inherit them.

\begin{lemma}\label{lem:inheritance-gj-lj-main}
Let $X$ be a Borel space, $\Theta$ a Hilbert space, $D \subset \Theta$ open, $\mu \in \Prob(X)$ and $N : X \times \Theta \to \R^l$ such that $N(x,\theta)$ is measurable in $x$ and Fr\'echet differentiable in $\theta$ for all $(x,\theta) \in X \times \Theta$ and the integral $\int \Vert N(x,\theta) \Vert^2 d\mu(x)$ exists and is finite for all $\theta \in \Theta$. Suppose that there exists $\hat{K}_N, \hat{L}_N : X \to \R$ with $K_N = \sqrt{\int \hat{K}_N^2 d\mu}$ and $L_N = \sqrt{\int \hat{L}_N^2 d\mu}$ both finite such that for $\mu$-almost every $x\in X$, $N(x,\cdot) : \Theta \to \R^l$ is $\hat{K}_N(x)$-BJ and $\hat{L}_N(x)$-LJ on $D$. 

Then $N_\mu : \Theta \to L^2(\mu,\R^l)$ is differentiable on $D$. For all $\theta \in D$, the Jacobian $\partial N_\mu(\theta) \in \Ell(\Theta,L^2(\mu,\R^l))$ is given by
\begin{equation}\label{eq:jacobian-of-mu} \partial N_\mu(\theta) \eta = \partial_\theta N(\cdot,\theta) \eta
\end{equation}
for all $\eta \in \Theta$. The adjoint Jacobian ${\partial N_\mu(\theta)}^* \in \Ell(L^2(\mu,\R^l),\Theta)$ is given by
\[
{\partial N_\mu(\theta)}^* f = \int {\partial_\theta N(x,\theta)}^* f(x) d\mu(x)
\]
for all $f \in L^2(\mu,\R^l)$. And $N_\mu$ is $K_N$-BJ and $L_N$-LJ on $D$.
\end{lemma}

Note that the self-adjoint operator $\partial N_\mu(\theta) {\partial N_\mu(\theta)}^* \in \Ell(L^2(\mu,\R^l),L^2(\mu,\R^l))$ on the feature space $L^2(\mu,\R^l)$ given by
\[
\partial N_\mu(\theta) {\partial N_\mu(\theta)}^* f = \int \partial_\theta N(\cdot,\theta) {\partial_\theta N(x,\theta)}^* f(x) d\mu(x)
\]
for $f \in L^2(\mu,\R^l)$ is exactly the NTK at $\theta$. The lemma below characterizes its block matrix representation over \emph{empirical measures}, i.e., $\mu=\frac{1}{d}\sum_{i=1}^d \delta_{x_i}$.

\begin{lemma}\label{lemma_ntk_block_matrix}
If the dataset $\mu = \frac{1}{d} \sum_{i=1}^d \delta_{x_i}$ is an empirical measure of samples $\{x_1,\cdots,x_d\} \subset X$, then $L^2(\mu,l) \cong \R^{dl}$ and the NTK operator $\partial N_\mu(\theta) {\partial N_\mu(\theta)}^* \in \Ell(L^2(\mu,l),L^2(\mu,l))$ has the block matrix representation
\[
\left[\frac{1}{d} \partial_\theta N(x_i,\theta) {\partial_\theta N(x_j,\theta)}^* : i,j \in [1:d]\right] \in \R^{dl \times dl}.
\]
\end{lemma}

The induced mapping $N_\mu$ being $\lambda_N$-UC and $K_N$-BJ means that the spectrum of the NTK is contained in $[\lambda_N, K_N^2]$. In the case of finite data, i.e., $\vert \support(\mu) \vert = d \in N_+$, one has that $\dim(L^2(\mu,\R^l)) = dl$, and the NTK being $\lambda_N$-coercive reduces to its matrix representation being positive definite with its smallest eigenvalue bounded from below by $\lambda_N$. In the case of infinite data, $\lambda_N$-coercivity is a stronger condition than positive definiteness, since the descending eigenvalues of a positive definite operator can converge to $0$, not having a uniform lower bound. In any case, having $\dim(\Theta) \geq \dim(L^2(\mu,\R^l))$, i.e., overparameterization, is a necessary condition for $N_\mu$ being $\lambda_N$-UC with $\lambda_N>0$. For an initial parameter $\theta_0 \in \Theta$, we need this condition to hold on a ball around $\theta_0$, but by Lemma~\ref{lemma_lazy_training} if the LJ constant of $N_\mu$ is sufficiently small it is enough to have coercivity at the initial point (i.e., lazy training can be exploited).

\subsection{Multilayer Perceptron}\label{mlp}

We will now introduce a multilayer perceptron (MLP) that, for sufficiently wide hidden layers, satisfies the hypotheses of Lemma~\ref{lem:inheritance-gj-lj-main} with large probability. Let $J>1$, $X=\R^k$ and $\Theta = \prod_{i=1}^J \R^{m_{i-1} \times m_i} \times \R^{m_i}$ with $m_0=k$, $m_J=l$ and $m_i = \gamma_i m$ for $i \in [1:J-1]$ for $m \in \N_+$ and some fixed $\{\gamma_1,\cdots,\gamma_{J-1}\} \subset \N_+$, equipped with the norm $\|\theta\|=\sqrt{\|A_1\|_F^2+\|b_1\|^2+\cdots+\|A_J\|_F^2+\|b_J\|^2}$ for any $\theta = \theta_{1 : J} = (A_1,b_1,\cdots,A_J,b_J) \in \Theta$. Let $\phi : \R \to \R$ be differentiable with $\Vert \phi' \Vert_\infty$ and $\Vert \phi' \Vert_L$ both finite and $|\phi(x)| \leq \vert x \vert$.\footnote{Note that standard smooth rectifiers including GELU, Softplus$-\log(2)$, SiLU/Swish and Mish satisfy this.} Given any $x \in X$ and $\theta \in \Theta$, let a $J$-layer MLP $N=N_J : \R^k \times \Theta \to \R^l$ be defined recursively as \[N_1(x,\theta_1) = A_1 x + b_1\] and
\[
N_i(x,\theta_{1:i}) = A_i \phi\left(\frac{1}{\sqrt{m}} N_{i-1}(x,\theta_{1 : i-1})\right)+b_i
\]
for $i \in [2,J]$. Let $\theta_0 = (A_{0,1},b_{0,1},\ldots,A_{0,J},b_{0,J}) \in \Theta$ be a random vector with each coordinate distributed according to a standard normal $\mathcal{N}(0,1)$, except for those of $A_{0,J}$ distributed according to $\mathcal{N}(0,\frac{1}{m})$. In the following, the $O(\cdot)$ and $\Omega(\cdot)$ notations are understood for sufficiently large $m$.

\begin{remark}
Note that this is (up to a constant factor) equivalent to the standard LeCun or He initialization. However, including the factor $\sqrt{m}$ in the architecture rather than in the initialization (except for the last layer) makes an important difference when training, as the updates with gradient descent then are automatically ``of the right order''. This will ultimately lead to lazy training and learning rate transfer. 
\end{remark}

For most applications in machine learning, it is usually assumed that the dataset follows some distribution $\nu \in \Prob(\R^k)$ and we have some samples of it $x_1,\ldots,x_d\in \R^k$ with which we form the empirical measure $\mu:=\frac{1}{d}\sum_{i=1}^d \delta_{x_i}$. We assume (as commonly done in the literature, see \citep[Assumption~2.2]{Nguyenetal2021} and \citep[Assumption~2]{Bombarietal2022} and note that the two-sided bound follows from the one-sided bound bellow) that $\nu$ satisfies the \emph{Lipschitz (or Gaussian) concentration property}, meaning that for an absolute constant $c_\nu>0$, any $t>0$ and any Lipschitz continuous $g : \R^k \to \R$, one has with probability at most $e^{-\frac{c_\nu t^2}{\Vert g \Vert_L^2}}$ that $g(x) - \int g d\nu \geq t$. We additionally assume that the moments
\[
M_{\nu,i} = \int \Vert \cdot \Vert^i d\nu
\]
of $\nu$ are finite for $i \in [1,2(J-1)]$. The following lemma shows that with high probability, in a neighborhood of the initial parameter, the BJ and LJ constants of the map $N_\mu : \Theta \to L^2(\mu,\R^l)$ are controlled by the even order moments of the data generating distribution $\nu$ with high probability, with the concentration getting stronger as we sample more data.\footnote{A variant of the result, without reference to a data generating distribution, is Theorem~\ref{theorem_mlp}.} In particular, the order of the BJ constant is controlled by the second moment, while the order of the LJ constant is controlled by the even order moments up to $2(J-1)$. Additionally, the order of the LJ constant depends inversely on the square root $\sqrt{m}$ of the width. Increasing the width results in the probability getting higher, as well as the LJ constant decreasing. The latter effect leads to lazy training, which can be exploited via Lemma~\ref{lemma_lazy_training}.

\begin{theorem}\label{theorem_mlp_empirical}
Fix any $\epsilon_K, \epsilon_L >0$ and any $C>0$. Let $\theta_0$ be chosen randomly as described above and define $D= B(\theta_0, C\sqrt{m}) \subset \Theta$. 

Then, with probability at least
\begin{multline}\label{eq:prob-LJ-BJ}
1 - 4Je^{-\Omega_{\gamma_{1:J-1},k,l}(m)}-2e^{-c_1 d \min\left( \frac{\epsilon_K^2}{C_{\nu,1}^2}, \frac{\epsilon_K}{C_{\nu,1}} \right)}\\
 - 2e^{-c_J \min\left( \frac{d \epsilon_L^2}{C_{\nu,J}}, \left(\frac{d \epsilon_L}{C_{\nu,J}}\right)^{\frac{1}{J-1}} \right)}
\end{multline}
with absolute constants $c_1,c_J,C_{\nu,1},C_{\nu,J}$, the induced mapping $N_\mu$ is $K_N$-LJ and $L_N$-LJ on $D$ with
\[
K_N = O\left(\sqrt{M_{\nu,2}+1+\epsilon_K}\right)
\]
and
\[
L_N = O\left(\frac{1}{\sqrt{m}} \sqrt{\sum_{i=0}^{J-1}\binom{J-1}{i}M_{\nu,2(J-1-i)}+\epsilon_L}\right),
\]
and
\[
\Vert N_\mu(\theta_0) \Vert = O\left( \sqrt{M_{\nu,2}+1+\epsilon_K} \right).
\]
Where these last three implicit constants depend on $C,\gamma_{1:L-1},k,l,\|\phi'\|_\infty,\Vert \phi' \Vert_L$.
\end{theorem}
\begin{proof}
By Theorem~\ref{thm:nn-map-bj} and Theorem~\ref{thm:nn-map-lj} (which condition on the same event) we have that $N(x,\cdot)$ is $O(\sqrt{\Vert x \Vert^2 + 1})$-BJ and $O(\frac{1}{\sqrt{m}}\sqrt{\Vert x \Vert^2 + 1}^{J-1})$-LJ. These facts combined with Lemma~\ref{lem:inheritance-gj-lj-main} give us BJ and LJ bounds in terms of the moments of $\mu$. These are governed by those of $\nu$ by Corollary~\ref{cor:concentration-sub-weibull}, which gives the first two claims of the result. By Lemma~\ref{lem:nn-map-bounded-initial} (which also conditions on the same event) we have an estimate of $\|N(x,\theta_0)\|$ for every $x$. Integrating over $\mu$ gives a bound in terms of the second moment of $\mu$, which, when combined with Corollary~\ref{cor:concentration-sub-weibull}, gives us the last claim. As all these events happen together with probability at least \eqref{eq:prob-LJ-BJ} the result follows.
\end{proof}

\subsection{Convergence of Overparameterized Learning}\label{convergence}

We need two more ingredients in order to apply Theorem~\ref{theorem_gd_f_circ_F} to the prototype problem. First, we make an assumption about the concentration of the smallest eigenvalue of the NTK at initialization.

\begin{assumption}[Concentration of the smallest eigenvalue of the NTK at initialization]\label{assumption_lambda_min}
Suppose that we have a model such as the one described in Section~\ref{mlp}. Then there exists $\epsilon_\lambda \geq 0$ such that with probability at least $1-\epsilon_\lambda$ with respect to sampling $\theta_0$ and $\mu$ we have
\[
\lambda_{\min}(\partial N_\mu(\theta_0) {\partial N_\mu(\theta_0)}^*) = \Omega(1)
\]
where the implicit constant and $\epsilon_\lambda$ may depend on $\nu,d,J,k,l,\gamma_{1:J-1},\phi$.
\end{assumption}

\begin{remark}
Note that this assumption does not impose $\epsilon_\lambda$ decreasing as width is increased. However this has been shown in several special cases for less general MLPs than ours, such as by \citet[Theorem~3.2]{Montanarietal2020}, \citet[Theorem~4.1]{Nguyenetal2021}, \citet[Theorem~2.1]{Wangetal2021} and \citet[Theorem~1]{Bombarietal2022}. Also in Subsection~\ref{experiment_lambda_min} we experimentally show that the smallest eigenvalue separates from 0 as $m$ increases in our model.
\end{remark}

Finally, let $f : \R_+ \to \R_+$ be any function such that $\lim_{m\to\infty}f(m)\sqrt{m}=\infty$ and $\lim_{m\to\infty}f(m)=0$. The theorem below shows that for sufficient width, we have, with high probability, the following. There is convergence at a linear rate to a global optimum and implicit regularization, i.e., the model interpolates the data, convergence is fast and the global optimum found by GD is very close to both the initial parameter and the global optimum that is closest to initialization. In the following theorem and its proof, the constants corresponding to the moments of the data generating distribution are suppressed. A more detailed proof can be found in Appendix~\ref{app:mlp}.\footnote{Note that in supervised learning, the Lipschitz assumption on $\nabla_z \iota(\cdot,z)$ is satisfied for least squares loss if the target function $t$ is Lipschitz, and for the regularized classification loss if $\inf_{x_1, x_2 \in \R^k : t(x_1) \neq t(x_2)} \Vert x_1 - x_2 \Vert > 0$.}

\begin{theorem}[Convergence of GD in overparameterized learning] \label{thm:general-conv-mlp}
Let $N(x,\theta)$ be a $J$-layer MLP as defined in Section~\ref{mlp}. Let $\nu\in \mathcal{P}(\R^k)$ be a probability distribution satisfying the Lipschitz concentration property and let $\mu=\frac{1}{d}\sum_{i=1}^d \delta_{x_i}$ be an empirical measure obtained by sampling $d$ independent elements $\{x_1,\cdots,x_d\}$ from $\nu$. Fix any $\epsilon_K,\epsilon_L,\epsilon_\Ell>0$. Suppose that the Assumption~\ref{assumption_lambda_min} holds with parameter $\epsilon_\lambda$. Let $\Ell_\mu$ be a loss functional induced by an integrand $\iota : \R^k \times \R^l \to \R$ such that $\iota(x,\cdot)_* = \iota_* \in \R$ and $\iota(x,\cdot)$ is $L_{\Ell}$-LG and $\lambda_{\Ell}$-PL globally for all $x \in X$ with $\nabla_z \iota(\cdot,z) : \R^k \to \R^l$ being $L_\Ell'$-Lipschitz for all $z \in \R^l$. Suppose that we choose the random initial parameter $\theta_0 \in \Theta$ with $\mathcal{N}(0,1)$ independently at each entry, except for those of $A_{0,J}$ distributed according to $\mathcal{N}(0,\frac{1}{m})$.

Denote by $T$ the set of variables $\iota,\nu,\phi,l,k,f,\{\gamma_i : i \in [1:J-1]\},\{M_{\nu,2i} : i \in [1:J-1]\},\epsilon_K,\epsilon_L,\epsilon_\Ell$.
Then there exists $M=M_{T}>0$ such that if $m\ge M$ then with probability at least 
\begin{multline}\label{probability_convergence}
1
- \epsilon_\lambda
- 4Je^{-\Omega_{\gamma_{1:J-1},k,l}(m)} 
- 2e^{-c_1 d \min\left( \frac{\epsilon_K^2}{C_{\nu,1}^2}, \frac{\epsilon_K}{C_{\nu,1}} \right)} \\
- 2e^{-c_J \min\left( \frac{d \epsilon_L^2}{C_{\nu,J}}, \left(\frac{d \epsilon_L}{C_{\nu,J}}\right)^{\frac{1}{J-1}} \right)}
- 2e^{-c_1 d \min\left( \frac{\epsilon_\Ell^2}{C_{\nu,L_\Ell'}^2}, \frac{\epsilon_\Ell}{C_{\nu,L_\Ell'}} \right)}
\end{multline}
with absolute constants $c_1,c_J,C_{\nu,1},C_{\nu,J},C_{\nu,L_\Ell'}$, there exists $L=O_{T}(1)$ such that we can choose a learning rate $\alpha\in\left(0,\frac{2}{L}\right)$ and have $q=q_{T,\alpha} \in (0,1)$. Define GD starting at $\theta_0$ recursively for $i \geq 1$ as
\[
\theta_i = \theta_{i-1} - \alpha \nabla (\Ell_\mu \circ N_\mu)(\theta_{i-1}).
\]
Then convergence is linear with rate $q$, i.e., for all $i \ge 0$
\[
(\Ell_\mu \circ N_\mu) (\theta_i) - {\Ell_\mu}_* 
\leq q^i ((\Ell_\mu \circ N_\mu) (\theta_0) - {\Ell_\mu}_*).
\]
The sequence $\theta_i$ converges to an interpolating solution  $\theta_* = \lim_{i \to \infty} \theta_i$, i.e.,
\[
(\Ell_\mu \circ N_\mu)(\theta_*) = {\Ell_\mu}_*.
\]
And there is implicit regularization, i.e.,
\begin{equation} \label{implicit_regularization_bound}
\Vert \theta_* - \theta_0 \Vert \leq R
\end{equation}
with $R=O_{T,\alpha}(1)$.

Moreover, if $\hat{\theta}_* = \argmin_{\theta \in \Theta : (\Ell_\mu \circ N_\mu)(\theta) = {\Ell_\mu}_*}{ \Vert \theta - \theta_0 \Vert }$, i.e., if $\hat{\theta}_* \in \Theta$ is an optimum of $(\Ell_\mu \circ N_\mu)$ that is closest to $\theta_0$, then one has $
\Vert \theta_* - \theta_0 \Vert =O_{T,\alpha}(\Vert \hat{\theta}_* - \theta_0 \Vert)$.
\end{theorem}
\begin{proof}[Proof sketch.]
By Theorem~\ref{theorem_mlp_empirical}, we have that with high probability $N_\mu$ is $K_N$-BJ and $L_N$-LJ on $D=B(\theta_0,\sqrt{m}f(m))$ with $K_N=O_T(1)$ and $L_N=O_T(\frac{1}{\sqrt{m}})$ and $\Vert N_\mu(\theta_0) \Vert = O_T(1)$. Note that the conditions on $\iota$ imply that $\Ell_\mu(0) = O_T(1)$ and $\Vert \nabla \Ell_\mu(0) \Vert = O_T(1)$ with high probability. By $\Vert N_\mu(\theta_0) \Vert = O_T(1)$ and Lemma~\ref{lemma_initial_loss_value}, we have that $\Ell_\mu(N_\mu(\theta_0)) = O_T(1)$. Define $K_\Ell = \sqrt{2L_\Ell(\Ell_\mu(N_\mu(\theta_0)) - {\Ell_\mu}_*)}$, so that $K_\Ell=O_T(1)$ as well (since  and ${\Ell_\mu}_* = \iota_*$). By Assumption~\ref{assumption_lambda_min}, we have $\lambda_{\min}(\partial N_\mu(\theta_0) {\partial N_\mu(\theta_0)}^*) = \Omega_T(1)$ with probability $1-\epsilon_\lambda$. Via Lemma~\ref{lemma_lazy_training} we exploit lazy training and have that $N_\mu$ is $\lambda_N$-UC on $D$ with $\lambda_N = \lambda_{\min}(\partial N_\mu(\theta_0) {\partial N_\mu(\theta_0)}^*) - 2 K_N L_N \sqrt{m} f(m) = \Omega_T(1)$ for $m$ large enough (as $2 K_N L_N \sqrt{m} f(m)=O_T(f(m))$ and $\lim_{m\to \infty}f(m)= 0$). We can now define the constants of Theorem~\ref{theorem_gd_f_circ_F}. Let $K=K_N K_\Ell=O_T(1)$, $L=K_N^2 L_\Ell + K_\Ell L_N=O_T(1)$, $\lambda = \lambda_N \lambda_\Ell=\Omega_T(1)$, choose a learning rate $\alpha \in (0,\frac{2}{L})$, and let $q = 1 + L\lambda\alpha^2-2\lambda\alpha \in (0,1)$ and $R=\frac{\alpha K}{1-\sqrt{q}}$. We have $q=O_{T,\alpha}(1)$, so that $R=O_{T,\alpha}(1)$ as well. Therefore, for $m$ large enough, we have that $R < \sqrt{m}f(m)$ (since $\lim_{m \to \infty} \sqrt{m}f(m)=\infty$), so that the conditions of Theorem~\ref{theorem_gd_f_circ_F} are satisfied and the proof is complete.
\end{proof}

As before, the optimal learning rate is $\alpha_* = \frac{1}{L}$. The proof above shows that $L=O(1)$, so that $\alpha_* = \Omega(1)$. Even though the optimal learning rate can slightly increase with increasing width, this leads to learning rate transfer as in \citet{Yangetal2022}, as demonstrated in Subsection~\ref{experiment_lr_transfer}. Although not made explicit in our assumption, the concentration of the smallest eigenvalue of the NTK at initialization gets stronger with increasing width, which has been shown to happen in other works cited above, and is demonstrated for our MLP formulation in Subsection~\ref{experiment_lambda_min}. The fact $L_N=O(\frac{1}{\sqrt{m}})$ means that training ``gets lazier'' with increasing width, as demonstrated in Subsection~\ref{experiment_lazy_training}. As can be seen from the proof above, the latter two effects lead to the convergence rate $q$ decreasing (therefore convergence getting faster) as width is increased. In turn, the GD path length bound $R$ decreases as well, making the implicit regularization effect stronger (and the global optimum found by GD gets closer and closer to the one closest to initialization). By the following lemma (proved in Appendix~\ref{app:mlp}), this leads to the Lipschitz constant $\Vert N(\cdot,\theta_*) \Vert_L$ of the trained MLP to be smaller and smaller as well, courtesy of the bound $\Vert N(\cdot,\theta) \Vert_L \leq \left(\frac{\Vert \phi \Vert_L}{\sqrt{m}}\right)^{J-1} \prod_{j=1}^J \Vert A_j \Vert$. This is demonstrated in Subsection~\ref{experiment_implicit regularization}.

\begin{lemma}\label{lemma_lip_const}
With the assumptions of Theorem~\ref{thm:general-conv-mlp}, the solution $\theta_*$ found by GD in the event with probability at least \eqref{probability_convergence} is such that $\Vert N(\cdot,\theta_*) \Vert_L = O_{T,\alpha}(1)$.
\end{lemma}

\subsection{Generalization Bounds}\label{generalization}

In overparameterized learning, it is desirable for the solution to generalize to the real world, i.e., for the generalization error $\Ell_\nu \circ N_\nu(\theta_*) - \Ell_\mu \circ N_\mu(\theta_*)$ to be small. It turns out that the PL condition, Lipschitz concentration and implicit regularization lead to bounds on the generalization error.

\begin{theorem}[Generalization in overparameterized learning]\label{theorem_generalization}
Fix any $\epsilon_\nu>0$. With the assumptions of Theorem~\ref{thm:general-conv-mlp}, the solution $\theta_*$ found by GD in the event with probability at least \eqref{probability_convergence} is such that with probability at least $1 - e^{-c_1 d \min\left( \frac{\epsilon_\nu^2}{C_T^2}, \frac{\epsilon_\nu}{C_T} \right)}$ with absolute constants $c_1,C_T$ one has that the generalization error is bounded as
\[
(\Ell_\nu \circ N_\nu)(\theta_*) - (\Ell_\mu \circ N_\mu)(\theta_*) \leq \epsilon_\nu.
\]
\end{theorem}
\begin{proof}
Let $g : \R^k \to \R$ be defined as $g(x) = \Vert \nabla_z \iota(x,N(x,\theta_*)) \Vert$, composing $N_*(x) = (x,N(x,\theta_*))$, $\nabla_z \iota$ and $\Vert \cdot \Vert$. By Lemma~\ref{lemma_lip_const}, $\Vert N(\cdot,\theta_*) \Vert_L = O_{T,\alpha}(1)$, so that $\Vert N_* \Vert_L \leq \sqrt{\Vert N(\cdot,\theta_*) \Vert_L^2 + 1} = O_{T,\alpha}(1)$. Also, $\Vert \nabla_z \iota \Vert_L \leq \sqrt{L_\Ell^2 + {L_\Ell'}^2}$ and $\Vert \cdot \Vert$ is $1$-Lipschitz, so that $\Vert g \Vert_L = O_{T,\alpha}(1)$. By Lipschitz concentration we have $\vert g(x) - \int g d\nu \vert > t$ with probability at most $2e^{-\frac{c_\nu t^2}{\Vert g \Vert_L^2}}$, i.e., $g(x)$ with $x$ distributed according to $\nu$ is sub-Gaussian with norm $O_{T,\alpha}(1)$. By \citet[Lemma~2.7.6]{Vershynin2018}, $g(x)^2$ is sub-exponential with norm $O_{T,\alpha}(1)$. Note that $\Vert \nabla \Ell_\mu(N_\mu(\theta_*)) \Vert^2 = \frac{1}{d} \sum_{i=1}^d g(x_i)^2 = 0$, $\int g^2 d\nu = \Vert \nabla \Ell_\nu(N_\nu(\theta_*)) \Vert^2$, ${\Ell_\nu}_* = {\Ell_\mu}_* = \iota_* = (\Ell_\mu \circ N_\mu)(\theta_*)$ by \citet[Theorem~3A]{Rockafellar1976} and $(\Ell_\nu \circ N_\nu)(\theta_*) - {\Ell_\nu}_* \leq \frac{1}{2\lambda_\Ell} \Vert \nabla \Ell_\nu(N_\nu(\theta_*)) \Vert^2$ since $\Ell_\nu$ is $\lambda_\Ell$-PL by Lemma~\ref{lem:cond-iota}. Hence via Lemma~\ref{sub_weibull_bernstein} the claim follows.
\end{proof}

Note that the bound $(\Ell_\nu \circ N_\nu)(\theta_*) - {\Ell_\nu}_* \leq \epsilon_\nu$ holds as well. A consequence of Theorem~\ref{thm:general-conv-mlp} and Theorem~\ref{theorem_generalization} is that the generalization error decreases as width is increased and/or as more data is sampled, as demonstrated in Subsection~\ref{experiment_generalization}.

\section{Experiments}\label{experiments}

This section provides experimental results supporting our theoretical claims and the NTK assumption. We have used the MNIST dataset as the data generating distribution $\nu \in \Prob(X)$ with $X = \R^k$ and $k=28^2$, sampling datasets $\mu = \frac{1}{d} \sum_{i=1}^d \delta_{x_i} \in \Prob(X)$ of different sizes $d$. The loss functional was induced by the regularized classification loss $\ell(y,z) = \log \sum e^z - z_y + \frac{\lambda_\Ell}{2} \Vert z \Vert^2$. We have used an MLP with $1$ hidden layer (i.e., $J=2$) and $\phi$ being a smooth rectifier. Additional details can be found in Appendix~\ref{app:experiments}.

\subsection{Learning Rate Transfer}\label{experiment_lr_transfer}

\begin{figure}[ht]
\vskip 0.2in
\begin{center}
\centerline{
\resizebox{.9\linewidth}{!}{
\begin{tikzpicture}
\pgfplotsset{cycle list/Dark2}
\begin{axis}[
    ymode = log,
    xlabel = $\log_2(\alpha)$,
    ylabel = $\E (\Ell_\mu \circ N_\mu)(\theta_{2000})$,
    ymax = 100,
    legend columns=2,
]
\addlegendentry{$m=128$}
\addplot[mycolor1] table [x=Step, y=Value, col sep=comma] {csvs/lr_transfer/m_128.csv};
\addlegendentry{$m=256$}
\addplot[mycolor2] table [x=Step, y=Value, col sep=comma] {csvs/lr_transfer/m_256.csv};
\addlegendentry{$m=512$}
\addplot[mycolor3] table [x=Step, y=Value, col sep=comma] {csvs/lr_transfer/m_512.csv};
\addlegendentry{$m=1024$}
\addplot[mycolor4] table [x=Step, y=Value, col sep=comma] {csvs/lr_transfer/m_1024.csv};
\addlegendentry{$m=2048$}
\addplot[mycolor5] table [x=Step, y=Value, col sep=comma] {csvs/lr_transfer/m_2048.csv};
\addlegendentry{$m=4096$}
\addplot[mycolor6] table [x=Step, y=Value, col sep=comma] {csvs/lr_transfer/m_4096.csv};
\addlegendentry{$m=8192$}
\addplot[mycolor7] table [x=Step, y=Value, col sep=comma] {csvs/lr_transfer/m_8192.csv};
\addlegendentry{$m=16384$}
\addplot[mycolor8] table [x=Step, y=Value, col sep=comma] {csvs/lr_transfer/m_16384.csv};
\end{axis}
\end{tikzpicture}
}
}
\caption{Expected loss ($y$ axis) after training for $2000$ GD steps with respect to $\log_2$ of the learning rate ($x$ axis).}
\label{plot_lr_transfer}
\end{center}
\vskip -0.2in
\end{figure}
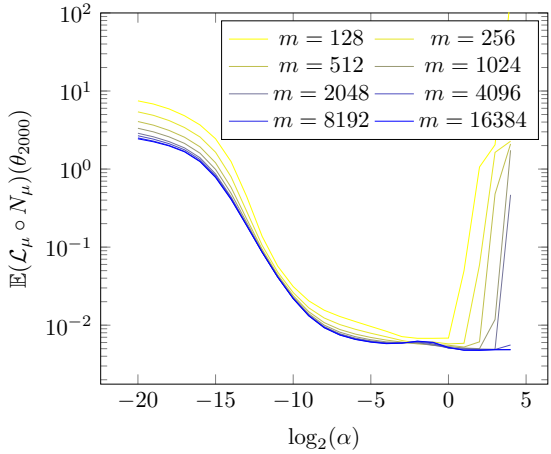

Figure~\ref{plot_lr_transfer} depicts the influence of the learning rate on the training loss after $2000$ iterations for different widths. The optimal learning rate has a slight drift, but is otherwise stable, similarly to the results of \citet[Figure~1]{Yangetal2022}. In this example one could optimize the learning rate for width $m=256$ and it would be approximately optimal for widths at least up to $m=16384$.

\subsection{Concentration of $\lambda_{\min}$ of the NTK at Initialization}\label{experiment_lambda_min}

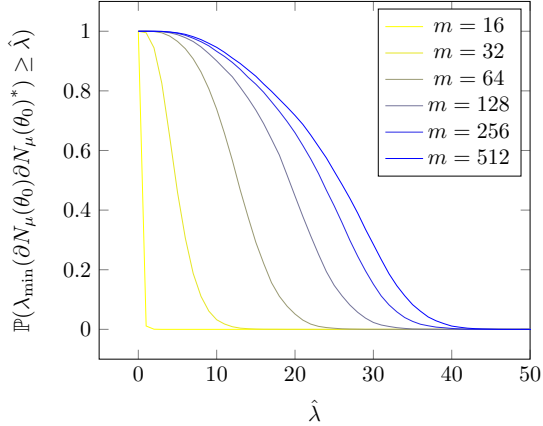
\begin{figure}[ht]
\vskip 0.2in
\begin{center}
\centerline{
\resizebox{.9\linewidth}{!}{
\begin{tikzpicture}
\pgfplotsset{cycle list/Dark2}
\begin{axis}[
    xlabel = $\hat{\lambda}$,
    ylabel = $\mathbb{P}( \lambda_{\min}(\partial N_\mu(\theta_0) {\partial N_\mu(\theta_0)}^*) \geq \hat{\lambda})$,
    legend columns=1,
    xmax=50,
]
\addlegendentry{$m=16$}
\addplot[mycolor1] table [x=Step, y=Value, col sep=comma] {csvs/lambda_min/m_16.csv};
\addlegendentry{$m=32$}
\addplot[mycolor2] table [x=Step, y=Value, col sep=comma] {csvs/lambda_min/m_32.csv};
\addlegendentry{$m=64$}
\addplot[mycolor4] table [x=Step, y=Value, col sep=comma] {csvs/lambda_min/m_64.csv};
\addlegendentry{$m=128$}
\addplot[mycolor5] table [x=Step, y=Value, col sep=comma] {csvs/lambda_min/m_128.csv};
\addlegendentry{$m=256$}
\addplot[mycolor7] table [x=Step, y=Value, col sep=comma] {csvs/lambda_min/m_256.csv};
\addlegendentry{$m=512$}
\addplot[mycolor8] table [x=Step, y=Value, col sep=comma] {csvs/lambda_min/m_512.csv};
\end{axis}
\end{tikzpicture}
}
}
\caption{Probability ($y$ axis) of the smallest eigenvalue of the NTK at initialization being above given thresholds ($x$ axis).}
\label{plot_lambda_min}
\end{center}
\vskip -0.2in
\end{figure}

In Figure~\ref{plot_lambda_min} we show how the concentration of the smallest eigenvalue of the NTK at initialization evolves as width is increased. The results align with Assumption~\ref{assumption_lambda_min}.

\subsection{Lazy Training}\label{experiment_lazy_training}

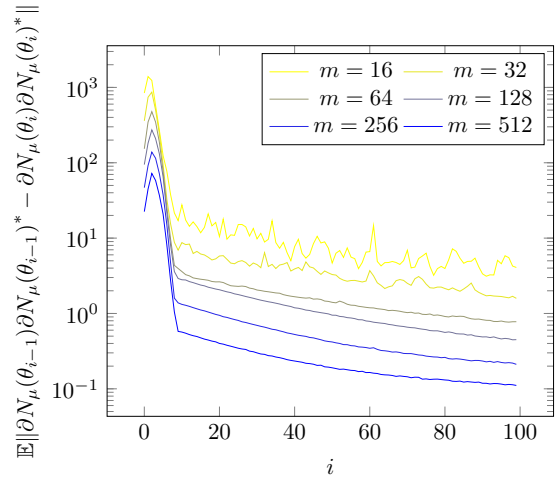
\begin{figure}[ht]
\vskip 0.2in
\begin{center}
\centerline{
\resizebox{.9\linewidth}{!}{
\begin{tikzpicture}
\pgfplotsset{cycle list/Dark2}
\begin{axis}[
    ymode = log,
    xlabel = $i$,
    ylabel = $\E \Vert \partial N_\mu(\theta_{i-1}) {\partial N_\mu(\theta_{i-1})}^* - \partial N_\mu(\theta_i) {\partial N_\mu(\theta_i)}^* \Vert$,
    legend columns=2,
]
\addlegendentry{$m=16$}
\addplot[mycolor1] table [x=Step, y=Value, col sep=comma] {csvs/lazy_training/m_16.csv};
\addlegendentry{$m=32$}
\addplot[mycolor2] table [x=Step, y=Value, col sep=comma] {csvs/lazy_training/m_32.csv};
\addlegendentry{$m=64$}
\addplot[mycolor4] table [x=Step, y=Value, col sep=comma] {csvs/lazy_training/m_64.csv};
\addlegendentry{$m=128$}
\addplot[mycolor5] table [x=Step, y=Value, col sep=comma] {csvs/lazy_training/m_128.csv};
\addlegendentry{$m=256$}
\addplot[mycolor7] table [x=Step, y=Value, col sep=comma] {csvs/lazy_training/m_256.csv};
\addlegendentry{$m=512$}
\addplot[mycolor8] table [x=Step, y=Value, col sep=comma] {csvs/lazy_training/m_512.csv};
\end{axis}
\end{tikzpicture}
}
}
\caption{Expected difference ($y$ axis) between consecutive NTKs during training steps ($x$ axis).}
\label{plot_lazy_training}
\end{center}
\vskip -0.2in
\end{figure}

In Figure~\ref{plot_lazy_training} we plot the influence of width on the operator norm of the difference of consecutive NTKs during training. As predicted by our theory, since $L_N = O(\frac{1}{\sqrt{m}})$, the NTK changes less and less as width is increased, i.e., overparameterization leads to lazy training.

\subsection{Implicit Regularization}\label{experiment_implicit regularization}

\begin{figure}[ht]
\vskip 0.2in
\begin{center}
\centerline{
\resizebox{.9\linewidth}{!}{
\begin{tikzpicture}
\pgfplotsset{cycle list/Dark2}
\begin{axis}[
    xlabel = $i$,
    ylabel = $\E \left(\frac{\Vert \phi \Vert_L}{\sqrt{m}}\right)^{J-1} \prod_{j=1}^J \Vert A_{i,j} \Vert$,
    legend columns=2,
]
\addlegendentry{$m=16$}
\addplot[mycolor1] table [x=Step, y=Value, col sep=comma] {csvs/implicit_regularization/m_16.csv};
\addlegendentry{$m=32$}
\addplot[mycolor2] table [x=Step, y=Value, col sep=comma] {csvs/implicit_regularization/m_32.csv};
\addlegendentry{$m=64$}
\addplot[mycolor3] table [x=Step, y=Value, col sep=comma] {csvs/implicit_regularization/m_64.csv};
\addlegendentry{$m=128$}
\addplot[mycolor4] table [x=Step, y=Value, col sep=comma] {csvs/implicit_regularization/m_128.csv};
\addlegendentry{$m=256$}
\addplot[mycolor5] table [x=Step, y=Value, col sep=comma] {csvs/implicit_regularization/m_256.csv};
\addlegendentry{$m=512$}
\addplot[mycolor6] table [x=Step, y=Value, col sep=comma] {csvs/implicit_regularization/m_512.csv};
\end{axis}
\end{tikzpicture}
}
}
\caption{Expected bound on $\Vert N(\cdot,\theta_i) \Vert_L$ ($y$ axis) during training steps ($x$ axis).}
\label{plot_implicit_regularization}
\end{center}
\vskip -0.2in
\end{figure}
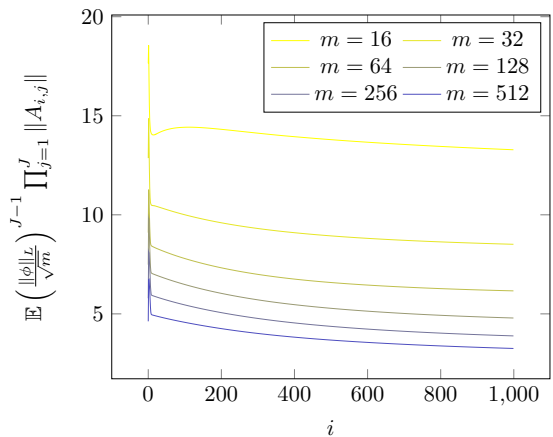

Figure~\ref{plot_implicit_regularization} shows the effect of width on the bound on $\Vert N(\cdot,\theta_i) \Vert_L$ during training. As we proposed, GD travels less and less as width is increased, implicitly Lipschitz regularizing $N(\cdot,\theta_i)$ more and more via Lemma~\ref{lemma_lip_const}. In other words, increasing overparameterization leads to stronger implicit regularization.

\subsection{Generalization Error}\label{experiment_generalization}

\begin{figure}[ht]
\vskip 0.2in
\begin{center}
\centerline{
\resizebox{\linewidth}{!}{
\begin{tikzpicture}
\begin{groupplot}[group style={group size= 1 by 3}, height=5cm, width=\linewidth]
\nextgroupplot[ylabel={$d=256$}, xmax = 2]
\addplot[mycolor1] table [x expr=\thisrowno{1}*0.01, y=Value, col sep=comma] {csvs/generalization/d_256_m_64.csv};
\addplot[mycolor2] table [x expr=\thisrowno{1}*0.01, y=Value, col sep=comma] {csvs/generalization/d_256_m_128.csv};
\addplot[mycolor3] table [x expr=\thisrowno{1}*0.01, y=Value, col sep=comma] {csvs/generalization/d_256_m_256.csv};
\addplot[mycolor4] table [x expr=\thisrowno{1}*0.01, y=Value, col sep=comma] {csvs/generalization/d_256_m_512.csv};
\addplot[mycolor5] table [x expr=\thisrowno{1}*0.01, y=Value, col sep=comma] {csvs/generalization/d_256_m_1024.csv};
\addplot[mycolor6] table [x expr=\thisrowno{1}*0.01, y=Value, col sep=comma] {csvs/generalization/d_256_m_2048.csv};
\addplot[mycolor7] table [x expr=\thisrowno{1}*0.01, y=Value, col sep=comma] {csvs/generalization/d_256_m_4096.csv};
\addplot[mycolor8] table [x expr=\thisrowno{1}*0.01, y=Value, col sep=comma] {csvs/generalization/d_256_m_8192.csv};
\coordinate (top) at (rel axis cs:0,1);
\nextgroupplot[ylabel={$d=512$}, xmax = 2]
\addplot[mycolor1] table [x expr=\thisrowno{1}*0.01, y=Value, col sep=comma] {csvs/generalization/d_512_m_64.csv};
\addplot[mycolor2] table [x expr=\thisrowno{1}*0.01, y=Value, col sep=comma] {csvs/generalization/d_512_m_128.csv};
\addplot[mycolor3] table [x expr=\thisrowno{1}*0.01, y=Value, col sep=comma] {csvs/generalization/d_512_m_256.csv};
\addplot[mycolor4] table [x expr=\thisrowno{1}*0.01, y=Value, col sep=comma] {csvs/generalization/d_512_m_512.csv};
\addplot[mycolor5] table [x expr=\thisrowno{1}*0.01, y=Value, col sep=comma] {csvs/generalization/d_512_m_1024.csv};
\addplot[mycolor6] table [x expr=\thisrowno{1}*0.01, y=Value, col sep=comma] {csvs/generalization/d_512_m_2048.csv};
\addplot[mycolor7] table [x expr=\thisrowno{1}*0.01, y=Value, col sep=comma] {csvs/generalization/d_512_m_4096.csv};
\addplot[mycolor8] table [x expr=\thisrowno{1}*0.01, y=Value, col sep=comma] {csvs/generalization/d_512_m_8192.csv};
\nextgroupplot[ylabel={$d=1024$}, xmax = 2, xlabel = $\epsilon$, legend style={at={(0.5,-0.4)},anchor=north}, legend columns=2]
\addlegendentry{$m=64$}
\addplot[mycolor1] table [x expr=\thisrowno{1}*0.01, y=Value, col sep=comma] {csvs/generalization/d_1024_m_64.csv};
\addlegendentry{$m=128$}
\addplot[mycolor2] table [x expr=\thisrowno{1}*0.01, y=Value, col sep=comma] {csvs/generalization/d_1024_m_128.csv};
\addlegendentry{$m=256$}
\addplot[mycolor3] table [x expr=\thisrowno{1}*0.01, y=Value, col sep=comma] {csvs/generalization/d_1024_m_256.csv};
\addlegendentry{$m=512$}
\addplot[mycolor4] table [x expr=\thisrowno{1}*0.01, y=Value, col sep=comma] {csvs/generalization/d_1024_m_512.csv};
\addlegendentry{$m=1024$}
\addplot[mycolor5] table [x expr=\thisrowno{1}*0.01, y=Value, col sep=comma] {csvs/generalization/d_1024_m_1024.csv};
\addlegendentry{$m=2048$}
\addplot[mycolor6] table [x expr=\thisrowno{1}*0.01, y=Value, col sep=comma] {csvs/generalization/d_1024_m_2048.csv};
\addlegendentry{$m=4096$}
\addplot[mycolor7] table [x expr=\thisrowno{1}*0.01, y=Value, col sep=comma] {csvs/generalization/d_1024_m_4096.csv};
\addlegendentry{$m=8192$}
\addplot[mycolor8] table [x expr=\thisrowno{1}*0.01, y=Value, col sep=comma] {csvs/generalization/d_1024_m_8192.csv};
\coordinate (bot) at (rel axis cs:1,0);
\end{groupplot}
\path (top-|current bounding box.west) -- node[anchor=south,rotate=90] {$\mathbb{P}( \Ell_\nu \circ N_\nu(\theta_{2000}) - \Ell_\mu \circ N_\mu(\theta_{2000}) \geq \epsilon)$} (bot-|current bounding box.west);
\end{tikzpicture}
}
}
\caption{Probability ($y$ axis) of the generalization error being above given thresholds ($x$ axis) after training for $2000$ GD steps.}
\label{plot_generalization}
\end{center}
\vskip -0.2in
\end{figure}
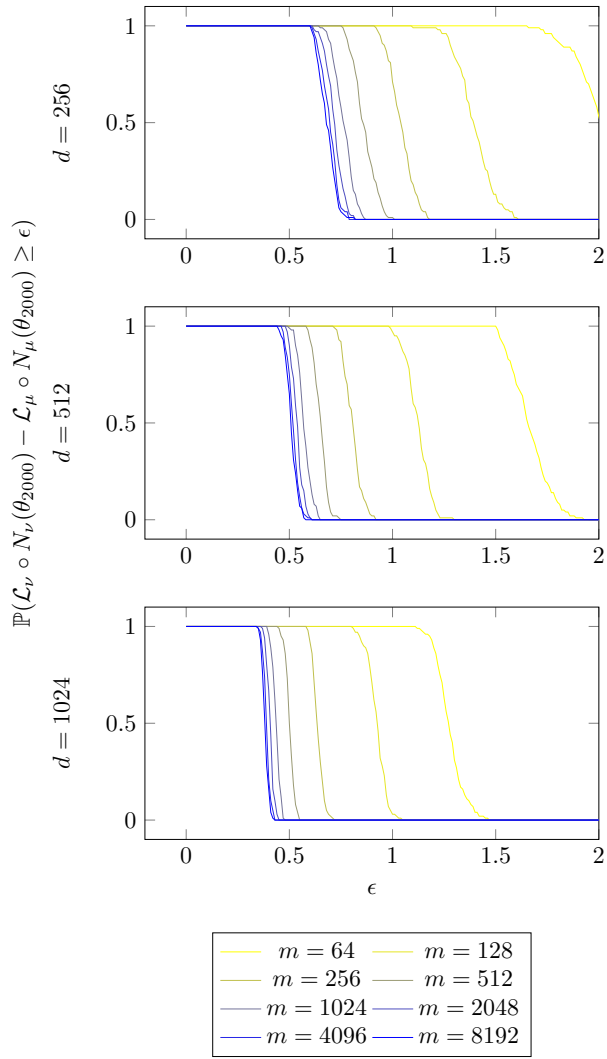

Courtesy of implicit regularization, the generalization error concentrates closer and closer to $0$ as the width and/or the number of data is increased, as shown in Figure~\ref{plot_generalization}. This explains the double descent phenomenon \citep{Belkinetal2019}. For $m$ sufficiently large, our convergence result holds and the model perfectly learns the data, but the implicit regularization effect is weak and the model does not generalize well. For even larger values of $m$ this effect gets stronger, leading to better generalization performance.

\section{Related Work} \label{related}

The main influences for our result on nonconvex composite optimization were \citet{Polyak1963} on the convergence of gradient descent for nonconvex losses, as well as \citet{Oymaketal2019a} and \citet{Liuetal2022} on the behavior of gradient descent on the nonlinear least squares problem. Theorem~\ref{theorem_gd_f_circ_F} is a significant generalization of the classical results of \citet{Polyak1963} (extending to the composite setting), as well as \citet[Theorem~2.1]{Oymaketal2019a} and \citet[Theorem~6]{Liuetal2022} (treating general losses). A less general result, without giving sharp constants or making the GD path length bounds explicit, was proposed by \citet{Songetal2021}, assuming that $G,H$ are finite dimensional (therefore not covering infinite dimensional parameter spaces and infinite data), a bound on $f(F(x_0))$, a global LJ bound for $F$ and that $f$ is twice differentiable.

Exploiting the lazy training phenomenon and the positivity of the NTK is a popular theoretical tool employed by contemporary works to prove convergence of GD for training neural networks. A number of works have focused on the least squares loss, with \citet{Duetal2018, Oymaketal2019a, Aroraetal2019, Wuetal2019, Oymaketal2020, Songetal2021} restricting to shallow neural networks and \citet{Duetal2019, Zouetal2019, Nguyenetal2020, Nguyen2021} treating deeper ones. Training deep neural networks with the binary classification loss was analyzed by \citet{Zouetal2018}. A more general work along this line is \citet{Allen-Zhuetal2019b}, treating supervised learning with general loss functions. A general property of these works is that the neural networks contain no biases. The only work among these to treat generalization is \citet{Aroraetal2019}. In contrast, we treat general loss functionals, covering learning problems outside of supervised learning, our MLP formulation has biases and we give bounds on the generalization error.

The NTK was proposed by \citet{Jacotetal2018} as a theoretical tool to analyze neural network training. Recognizing its importance, many works have since been focusing on proving concentration results about the smallest eigenvalue of the NTK at initialization \citep{Montanarietal2020, Nguyenetal2021, Wangetal2021, Bombarietal2022}. The first to analyze the lazy training phenomenon was \citet{Chizatetal2019}, restricting attention to gradient flows. Implicit regularization was treated in \citet{Oymaketal2019a}, and was described in the review papers \citet{Belkin2021} and \citet{Bartlettetal2021} as an important phenomenon to understand the success of deep learning.

\section{Limitations and Future Directions} \label{limitations}
Immediate future works are to prove that for our MLP formulation the smallest eigenvalue of the NTK at initialization is $\Omega(1)$ with high probability as in Assumption~\ref{assumption_lambda_min}, and to extend our results to minibatch stochastic gradient descent. A limitation is that we require $N(x,\theta)$ to be differentiable in $\theta$, which excludes activations which are not everywhere differentiable, e.g., the ReLU. Perhaps the generalized Jacobian theory of \citet{Palesetal2007a} could cover any locally Lipschitz activation, but is outside the scope of this paper. Although many loss function(al)s can easily be regularized to satisfy the PL condition, two future directions are to develop general LG and PL regularization strategies, and to generalize our results to loss functionals with weaker assumptions. In particular, there are learning problems where $\Ell_\mu$ is not an integral functional. While our convergence result in Subsection~\ref{convergence} already covers many problems of interest, some problems that can already be translated to the prototype problem, including the ones in Appendix~\ref{app:prototype}, are left for future work.

\newpage
\subsection*{Acknowledgements}
D\'avid Terj\'ek is supported by the Hungarian National Excellence Grant 2018-1.2.1-NKP-00008 and by the Hungarian Ministry of Innovation and Technology NRDI Office within the framework of the Artificial Intelligence National Laboratory Program. Diego Gonz\'alez-S\'anchez is supported by projects KPP 133921 and Momentum (Lend\"{u}let) 30003 of the Hungarian Government.

\bibliography{affol}
\bibliographystyle{affol}

\newpage
\appendix
\onecolumn

\section{Nonconvex Composite Optimization}\label{app:composite}

In this section we prove Theorem~\ref{theorem_gd_f_circ_F}. We will essentially show that under the assumptions of Theorem~\ref{theorem_gd_f_circ_F} we can deduce that $(f\circ F)$ is LG and PL and thus we can use similar arguments as those in \citet{Polyak1963}. In order to do this, first we need to prove some consequences of the LG, PL, LJ and UC conditions. For any $x,y\in H$ let us denote by $[x,y]$ the segment joining them, i.e., $\{tx+(1-t)y:t\in[0,1]\}$.

\begin{proposition}[Fundamental theorem of calculus]\label{fundamental_theorem_of_calculus}
Let $G,H$ be Hilbert spaces, let $D \subset G$ and let $F : G \to H$ be a function which is $L$-LJ on $D$ for some $L>0$. Then\footnote{The integral here is the Bochner integral \citep[Definition~1.6.8]{Cobzasetal2019}.}
\[
F(y)-F(x) = \int_0^1 \partial F(x+t(y-x))(y-x) dt
\]
for all $x,y \in D$ such that $[x,y] \subset D$.
\end{proposition}
\begin{proof}
Let $f : [0,1] \to H$ be defined as
\[
F(x+t(y-x)).
\]
Thus
\[
\partial f(t)=\partial F(x+t(y-x))(y-x)
\]
for $t \in (0,1)$. Then we have that
\[
\Vert \partial f(t_1) - \partial f(t_2) \Vert 
= \Vert (\partial F(x+t_1(y-x))-\partial F(x+t_2(y-x)))(y-x) \Vert
\]
\[
\leq L \Vert (t_1-t_2)(y-x) \Vert\Vert y-x \Vert
= \vert t_1-t_2 \vert L \Vert y-x \Vert^2
\]
for $t_1,t_2 \in (0,1)$. Hence $\partial f$ is $L \Vert y-x \Vert^2$-Lipschitz on $(0,1)$. In particular, $f$ is $C^1$ on $(0,1)$. By the Kirszbraun theorem \citep[Theorem~4.2.3]{Cobzasetal2019}, $\partial f$ has a (Lipschitz) continuous extension to $[0,1]$. By the fundamental theorem of calculus \citep[Theorem~89]{Hajeketal2014} applied to $f$,
\[
f(1)-f(0) = F(y)-F(x) = \int_0^1 \partial f(t) dt 
= \int_0^1 \partial F(x+t(y-x))(y-x) dt.
\]
\end{proof}

In particular, if $f : H \to \R$ is $L$-LG on $D \subset H$ for some $L>0$ then
\[
f(y)-f(x) = \int_0^1 \langle \nabla f(x+t(y-x)), y-x \rangle dt
\]
for all $x,y \in D$ such that $[x,y] \subset D$.

The LG property has an interesting consequence that we will use in the sequel.

\begin{proposition}\label{lg_corollary_1}
Let $f : H \to \R$ be an $L$-LG function  on $D \subset H$ for some $L>0$. Then for all $ x,y \in D$ such that $[x,y] \subset D$,
\[
\left\vert f(y) - f(x) - \langle \nabla f(x), y - x \rangle \right\vert \leq \frac{L}{2} \Vert y - x \Vert^2.
\]
\end{proposition}
\begin{proof}
By Proposition~\ref{fundamental_theorem_of_calculus} we have that
\[
f(y) - f(x)
= \int_0^1 \langle \nabla f(x + t (y - x)),(y - x)\rangle  \;dt.
\]
Adding and subtracting $\nabla f(x)$ the previous formula equals
\[
\int_0^1 \langle \nabla f(x) - \nabla f(x) + \nabla f(x + t (y - x)),(y - x)\rangle \; dt
\]
\[
\leq \langle \nabla f(x),(y - x)\rangle  + \int_0^1 L \Vert t(y-x)\Vert \Vert y-x \Vert dt 
= \langle \nabla f(x),(y - x)\rangle  + \frac{L}{2} \Vert y - x \Vert^2.
\]
Similarly,
\[
f(x) - f(y) 
= \int_0^1 \langle \nabla f(y + t (x - y)),(x - y)\rangle  \;dt
= \int_0^1 \langle \nabla f(x) - \nabla f(x) + \nabla f(y + t (x - y)),(x - y)\rangle \; dt
\]
\[
\leq \langle \nabla f(x),(x - y)\rangle  + \int_0^1 L \Vert (1 - t)(x-y)\Vert \Vert x-y \Vert dt 
= \langle \nabla f(x),(x - y)\rangle  + \frac{L}{2} \Vert x - y \Vert^2,
\]
giving the desired conclusion.
\end{proof}

A consequence of Proposition~\ref{lg_corollary_1} is the following.

\begin{proposition} \label{lg_corollary_2}
Let $f : H \to \R$ be an $L$-LG function on $D \subset H$ for some $L>0$ and bounded from below on $H$. Then for any $x \in D$ such that $\left[x, x-\frac{1}{L}\nabla f(x)\right] \subset D$ we have that
\[
\frac{1}{2}\Vert \nabla f(x) \Vert^2
\leq L(f(x) - f_*).
\]
\end{proposition}
\begin{proof}
By Proposition~\ref{lg_corollary_1} we have
\[
f\left(x-\frac{1}{L}\nabla f(x)\right) - f(x)
\leq \left\langle \nabla f(x),-\frac{1}{L}\nabla f(x) \right\rangle  + \frac{L}{2} \left\Vert -\frac{1}{L}\nabla f(x) \right\Vert^2
= -\frac{1}{2L} \Vert \nabla f(x) \Vert^2.
\]
Thus
\[
\frac{1}{2}\Vert \nabla f(x) \Vert^2
\leq L\left(f(x) - f\left(x-\frac{1}{L}\nabla f(x)\right)\right)
\leq L(f(x) - f_*).
\]
\end{proof}

\begin{proposition} \label{composition_pl}
Let $G,H$ be Hilbert spaces, let $\lambda_F, \lambda_f > 0$ be some constants. Let $F : G \to H$ be $\lambda_F$-UC on $\{x\}$ and $f : H \to \R$ be $\lambda_f$-PL on $F(\{x\})$. Then for all $ x \in D$ the composition $(f \circ F) : G \to \R$ satisfies
\[
\frac{1}{2} \Vert \nabla (f \circ F)(x) \Vert^2
\geq \lambda_f \lambda_F ((f \circ F) (x) -  f_*).
\]
\end{proposition}
\begin{proof}
First note that for all $x\in G$ we have $\nabla (f\circ F)(x) = \partial F(x)^* \nabla f(F(x))$. Hence we have that
\begin{align*}
\frac{1}{2}\Vert \nabla (f \circ F)(x) \Vert^2 &
= \frac{1}{2}\Vert {\partial F(x)}^* \nabla f(F(x)) \Vert^2 \\
& = \frac{1}{2} \langle {\partial F(x)}^* \nabla f(F(x)), {\partial F(x)}^* \nabla f(F(x)) \rangle \\
& = \frac{1}{2} \langle \nabla f(F(x)), \partial F(x) {\partial F(x)}^* \nabla f(F(x)) \rangle \\
& \geq \lambda_F \frac{1}{2} \langle \nabla f(F(x)), \nabla f(F(x)) \rangle \\
& = \lambda_F \frac{1}{2} \Vert \nabla f(F(x)) \Vert^2
\geq \lambda_f \lambda_F ((f \circ F) (x) -  f_*).
\end{align*}
Where in the first inequality we have used that $F$ is $\lambda_F$-UC and in the second inequality that $f$ is $\lambda_f$-PL.
\end{proof}

Note that this proposition tells us that $(f\circ F)$ is ``almost'' PL. Next we are going to show that if $f$ and $F$ are regular enough, their composition is LG.

\begin{proposition} \label{composition_lg}
Let $G,H$ be Hilbert spaces, $D\subset G$ convex and let $K_F,L_F,K_f,L_f \ge  0$ be constants. And let $F : G \to H$ be $K_F$-BJ and $L_F$-LJ on $D \subset G$ and $f : H \to \R$ be $K_f$-BG and $L_f$-LG on $F(D)$. Then $(f \circ F) : G \to \R$ is $(K_F^2 L_f + K_f L_F)$-LG on $D$.
\end{proposition}
\begin{proof}
For all $x, y, z \in G$ such that $x,y \in D$ and $\Vert z \Vert = 1$,
\begin{align*}
&\langle z, {\partial F(x)}^* \nabla f(F(x)) - {\partial F(y)}^* \nabla f(F(y)) \rangle \\
&= \langle z, {\partial F(x)}^* (\nabla f(F(x)) - \nabla f(F(y))) - ({\partial F(y)}^* - {\partial F(x)}^*)(\nabla f(F(y))) \rangle\\
&\leq \Vert z \Vert \Vert {\partial F(x)}^* \Vert \Vert \nabla f(F(x)) - \nabla f(F(y)) \Vert + \Vert z \Vert \Vert {\partial F(y)}^* - {\partial F(x)}^* \Vert \Vert \nabla f(F(y)) \Vert\\
&\leq K_F L_f \Vert F(x) - F(y) \Vert + L_F \Vert x - y \Vert K_f \leq (K_F^2 L_f + K_f L_F) \Vert x - y \Vert.
\end{align*}
Where we have used the fact that $[x,y] \subset D$ (since $D$ is convex) implies $\Vert F(x) - F(y) \Vert \leq K_F \Vert x - y \Vert$ via Proposition~\ref{fundamental_theorem_of_calculus}. The proposition then follows since
\begin{align*}
&\Vert \nabla (f \circ F) (x) - \nabla (f \circ F) (y) \Vert
= \Vert {\partial F(x)}^* \nabla f(F(x)) - {\partial F(y)}^* \nabla f(F(y)) \Vert \\
&= \sup_{\Vert z \Vert = 1} \langle z, {\partial F(x)}^* \nabla f(F(x)) - {\partial F(y)}^* \nabla f(F(y)) \rangle.
\end{align*}\end{proof}

The next result tells us that LG functions are BG on bounded sets.

\begin{proposition}\label{app:bound-sets}
Let $H$ be a Hilbert space, let $D \subset H$ be a bounded and let $f : H \to \R$ be an $L_f$-LG function on $D$ for some $L_f > 0$. Then $f$ is BG on $D$. Moreover, if $D \subset \overline{B}(0,R)$ for some $R>0$, then $f$ is $(L_f R + \Vert \nabla f(0) \Vert)$-BG on $D$.
\end{proposition}
\begin{proof}
Since $\nabla f : H \to H$ is Lipschitz on $D$ and $\Vert \cdot \Vert : H \to \R$ is Lipschitz globally, $(\Vert \cdot \Vert \circ \nabla f) : H \to \R$ is Lipschitz on $D$ as well. Hence $(\Vert \cdot \Vert \circ \nabla f)$ is bounded on $D$, so that equivalently, $f$ is BG on $D$.

For the second claim, let $x \in D$ and note that by the reverse triangle inequality
\[
\Vert \nabla f(x) \Vert - \Vert \nabla f(0) \Vert 
\leq \Vert \nabla f(x) - \nabla f(0) \Vert
\leq L_f \Vert x - 0 \Vert,
\]
so that
\[
\Vert \nabla f(x) \Vert 
\leq L_f \Vert x \Vert + \Vert \nabla f(0) \Vert
\leq L_f R + \Vert \nabla f(0) \Vert.
\]
\end{proof}

Next we prove an analogue of Proposition~\ref{lg_corollary_1} for the composition $(f\circ F)$.

\begin{proposition} \label{composition_lg_corollary_1}
Let $G,H$ be Hilbert spaces, let $D \subset G$ be bounded, $x\in D$ a fixed point and $K_F,K_f,L_F,L_f \ge 0$ be some constants. And let $F : G \to H$ be $K_F$-BJ and $L_F$-LJ on $D$ and $f : H \to \R$ be $L_f$-LG on $F(D)$ with $\Vert \nabla f(F(x)) \Vert \leq K_f$. For any $y\in G$ such that $[x,y] \subset D$ we have that
\[
\vert (f \circ F)(y) - (f \circ F)(x) - \langle \nabla (f \circ F)(x), y - x \rangle \vert
\leq \frac{ K_F^2 L_f + K_f L_F }{2} \Vert y-x \Vert^2.
\]
\end{proposition}
\begin{proof}
By Proposition~\ref{app:bound-sets} $f$ is $\hat{K}_f$-BG on $[x,y]$ for some $\hat{K}_f \geq 0$. Hence, by Proposition~\ref{composition_lg}, $(f \circ F)$ is $K_F^2 L_f+\hat{K}_f L_F$-LG on $[x,y]$. Thus by Proposition~\ref{fundamental_theorem_of_calculus}, we have that
\[
(f \circ F)(y) - (f \circ F)(x)= \int_0^1 \langle {\partial F(x + t(y - x))}^* \nabla f(F(x + t(y - x))), y - x \rangle dt.
\]
Hence
\begin{align*}
&(f \circ F)(y) - (f \circ F)(x) \\ & = 
\int_0^1 \langle {\partial F(x + t(y - x))}^* \nabla f(F(x + t(y - x))) \pm {\partial F(x)}^* \nabla f(F(x)), y - x \rangle dt \\
&= \langle \nabla (f \circ F)(x), y - x \rangle + \int_0^1 \langle {\partial F(x + t(y - x))}^* (\nabla f(F(x + t(y - x))) - \nabla f(F(x))) \\
&+ ({\partial F(x + t(y - x))}^* - {\partial F(x)}^*) \nabla f(F(x)), y - x \rangle dt\\
&\leq \langle \nabla (f \circ F)(x), y - x \rangle + \int_0^1 (K_F^2 L_f \Vert t(y-x) \Vert + K_f L_F \Vert t(y-x) \Vert) \Vert y-x \Vert dt\\
&= \langle \nabla (f \circ F)(x), y - x \rangle + \frac{ K_F^2 L_f + K_f L_F }{2} \Vert y-x \Vert^2.
\end{align*}
The other inequality follows similarly.
\end{proof}

The next result gives a decay in the objective value of a gradient descent.

\begin{proposition} \label{composition_gradient_descent_step}
Let $G,H$ be Hilbert spaces, let $D \subset G$ be bounded, $x \in D$ and $K_F,L_F,\lambda_F,L_f,\lambda_f \ge 0$ be constants such that $\lambda_F \leq K_F^2$, $\lambda_f \leq L_f$. And let $F : G \to H$ be $K_F$-BJ and $L_F$-LJ on $D$. Let $f : H \to \R$ be $L_f$-LG and $\lambda_f$-PL on $F(D)$. Suppose that $\Vert \nabla f(F(x)) \Vert \leq K_f$. Let $L = K_F^2 L_f + K_f L_F$, $\lambda = \lambda_F \lambda_f$ (so that $\lambda \leq L$) and $\alpha \in (0,\frac{2}{L})$ and $q = 1 + L \alpha^2 \lambda - 2 \alpha \lambda$ (so that $q \in (0,1)$). If $[x, x^\alpha] \subset D$ and $F$ is $\lambda_F$-UC on $\{x\}$, then
\[
(f \circ F) (x^\alpha) - f_* 
\leq q ((f \circ F) (x) - f_*).
\]
\end{proposition}
\begin{proof}
By Proposition~\ref{composition_lg_corollary_1} and Proposition~\ref{composition_pl} we have
\[
(f \circ F) (x^\alpha) - (f \circ F) (x) 
\leq \langle \nabla (f \circ F) (x), -\alpha \nabla (f \circ F) (x) \rangle  + \frac{L}{2} \Vert -\alpha \nabla (f \circ F) (x) \Vert^2
\]
\[
= \left(\frac{L}{2} \alpha^2 - \alpha\right) \Vert \nabla (f \circ F) (x) \Vert^2 
\leq (L \alpha^2 \lambda - 2 \alpha \lambda) ((f \circ F) (x) - {f }_*).
\]
Where we have used the fact that $\alpha \in (0,\frac{2}{L})$ implies $\left(\frac{L}{2} \alpha^2 - \alpha\right) \leq 0$. Adding $(f \circ F) (x) - {f}_*$ to both sides yields
\[
(f \circ F) (x^\alpha) - {f }_* \leq (1 + L \alpha^2 \lambda - 2 \alpha \lambda) ((f \circ F) (x) - {f }_*).
\]
Note that the conditions $\lambda \leq L$ and $\alpha \in \left(0,\frac{2}{L}\right)$ ensure that $(1 + L \alpha^2 \lambda - 2 \alpha \lambda) \in [0,1)$.
\end{proof}

Now we combine these ingredients to prove the main result of this section.

\begin{proof}[Proof of Theorem~\ref{theorem_gd_f_circ_F}] 
The first step is to prove by induction on $i\ge 0$ the following statements.
\begin{equation}\label{eq:app-main-1}
f(F(x_i))-f_* \le q^i (f(F(x_0))-f_*),
\end{equation}
\begin{equation}\label{eq:app-main-2}
\|x_{i+1}-x_i\|\le \alpha q^{i/2} K,
\end{equation}
\begin{equation}\label{eq:app-main-3}
x_i\in \overline{B}(x_0,R)
\end{equation}
and
\begin{equation}\label{eq:app-main-4}
\|\nabla f(F(x_i))\|\le q^{i/2}K_f.
\end{equation}

First, the case $i=0$. Note that \eqref{eq:app-main-1} and \eqref{eq:app-main-3} holds trivially in this case. By Proposition~\ref{lg_corollary_2} applied to $f$ and the point $F(x_0)$ we have \eqref{eq:app-main-4}. Note that $f$ is globally $L_f$-LG and thus we can apply this proposition. Finally, note that $ \nabla (f \circ F) (x_{0}) = \partial F(x_{0})^*\nabla f(F(x_{0}))$. Hence $\Vert \nabla (f \circ F) (x_{0}) \Vert \leq \Vert \partial F(x_{0}) \Vert \Vert \nabla f(F(x_{0})) \Vert \leq K_F K_f = K$ and $\| x_1-x_0 \|=\alpha \Vert \nabla (f \circ F) (x_{0}) \Vert \le \alpha K$ proving \eqref{eq:app-main-2}.

Now we show that case $i$ implies case $i+1$. By induction we have that 
\[
\|x_{i+1}-x_0\|\le \sum_{j=0}^i \|x_{j+1}-x_j\|\le \sum_{j=0}^i \alpha q^{j/2} K \le \frac{\alpha K}{1-\sqrt{q}}.
\]
Thus $x_{i+1}\in \overline{B}(x_0,R)$ and we have \eqref{eq:app-main-3}. Moreover, by a similar argument it follows that the segment $[x_i,x_{i+1}]\subset \overline{B}(x_0,R)$ and thus by Proposition~\ref{composition_gradient_descent_step} and the induction hypothesis we have \eqref{eq:app-main-1}.

By Proposition~\ref{lg_corollary_2} applied to $f$ and $F(x_{i+1})$ we get that $\|\nabla f(F(x_{i+1}))\|\le \sqrt{2 L_f (f(F(x_{i+1}))-f_*)}$. Using \eqref{eq:app-main-1} this can be further bounded by $\|\nabla f(F(x_{i+1}))\|\le \sqrt{2 L_f q^{i+1}(f(F(x_{0}))-f_*)} = q^{\frac{i+1}{2}}K_f$ thus proving \eqref{eq:app-main-4}. This also implies that $\Vert \nabla (f \circ F) (x_{i+1}) \Vert \leq \Vert \partial F(x_{i+1}) \Vert \Vert \nabla f(F(x_{i+1})) \Vert \leq K_F q^{\frac{i+1}{2}} K_f = q^{\frac{i+1}{2}} K$, yielding \eqref{eq:app-main-2}.


Now let us show that the sequence $x_i$ converges as $i\to\infty$. For any $i_2>i_1$ by the triangle inequality
\[
\Vert x_{i_2} - x_{i_1} \Vert 
\leq \sum_{j = i_1}^{i_2 - 1} \Vert x_{j+1} - x_j \Vert  
\leq \alpha K \sum_{j = i_1}^{i_2 - 1} q^{\frac{j}{2}} 
\leq \alpha K q^{\frac{i_1}{2}} \sum_{j = 0}^{\infty}q^{\frac{j}{2}} 
= \frac{\alpha K q^{\frac{i_1}{2}}}{1 - \sqrt{q}}.
\]
And thus this sequence is Cauchy and it converges to some $x_* = \lim_{i \to \infty} x_i$. Moreover, for every $i\ge 0$ we have that
\[
\Vert x_* - x_i \Vert \leq \frac{\alpha K q^{\frac{i}{2}}}{1 - \sqrt{q}}
\]
so in particular $\Vert x_* - x_0\Vert \leq \frac{\alpha K }{1 - \sqrt{q}}$. Note that \eqref{eq:app-main-1} implies in particular that the infimum of $f$ equals the infimum of $(f\circ F)$ and it is attained at $x_*$. Furthermore, $\|\nabla f(F(x_*))\|=0$ (and similarly for any other $\hat{x}_*$ that attains the infimum of $(f\circ F)$).

Finally we apply Proposition~\ref{composition_lg_corollary_1} with $K_f =0$ to get
\[
(f \circ F)(x_0) - (f \circ F)(\hat{x}_*) 
\leq \langle \nabla (f \circ F)(\hat{x}_*), x_0 - \hat{x}_* \rangle + \frac{K_F^2 L_f}{2} \Vert x_0 - \hat{x}_* \Vert^2 
= \frac{K_F^2 L_f }{2} \Vert x_0 - \hat{x}_* \Vert^2,
\]
so that
\[
K_F^2 L_f \Vert x_0 - \hat{x}_* \Vert \geq K_F \sqrt{2 L_f((f \circ F)(x_0) - {f}_*)} = K,
\]
which can be substituted into the bound on $\Vert x_* - x_0 \Vert$, giving the last claim of the theorem.
\end{proof}

Now we prove Lemma~\ref{lemma_lazy_training} concerning lazy training.
\begin{proof}[Proof of Lemma~\ref{lemma_lazy_training}]
For all $x \in B(x_0,R)$ and $y \in H$, one has
\[
\langle y, \partial F(x) {\partial F(x)}^*, y \rangle
= \langle y, (\partial F(x) {\partial F(x)}^* \pm \partial F(x_0) {\partial F(x_0)}^*) y \rangle
\]
\[
=\langle y, \partial F(x_0) {\partial F(x_0)}^* y \rangle + \langle y, (\partial F(x) {\partial F(x)}^* - \partial F(x_0) {\partial F(x_0)}^*) y \rangle
\]
\[
\geq (\lambda_0 - \Vert \partial F(x) {\partial F(x)}^* - \partial F(x_0) {\partial F(x_0)}^* \Vert) \Vert y \Vert^2
\geq (\lambda_0 - 2 K_F L_F R) \Vert y \Vert^2,
\]
where we used the fact that
\[
\Vert \partial F(x) {\partial F(x)}^* - \partial F(x_0) {\partial F(x_0)}^* \Vert
\leq \Vert \partial F(x) {\partial F(x)}^* - \partial F(x_0) {\partial F(x_0)}^* \pm \partial F(x) {\partial F(x_0)}^* \Vert
\]
\[
\leq \Vert \partial F(x) \Vert \Vert {\partial F(x)}^* - {\partial F(x_0)}^* \Vert + \Vert \partial F(x) - \partial F(x_0) \Vert \Vert {\partial F(x_0)}^* \Vert
\]
\[
\leq 2 K_F L_F \Vert x - x_0 \Vert
\leq 2 K_F L_F R.
\]
If one has $\lambda_F = \lambda_0 - 2 K_F L_F R > 0$, then this is equivalent to $\partial F(x) {\partial F(x)}^*$ being $\lambda_F$-coercive. Since this holds for any $x \in B(x_0,R)$, it follows that $F$ is $\lambda_F$-UC on $B(x_0,R)$.
\end{proof}

The final proof of this section is that of Lemma~\ref{lemma_initial_loss_value} about bounding the loss value on a bounded set.

\begin{proof}[Proof of Lemma~\ref{lemma_initial_loss_value}]
Via the reverse triangle inequality and the LG property, for any $x \in \overline{B}(0,R)$ one has
\[
\Vert \nabla f(x) \Vert - \Vert \nabla f(0) \Vert \leq \Vert \nabla f(x) - \nabla f(0) \Vert \leq L_f \Vert x - 0 \Vert,
\]
so that
\[
\Vert \nabla f(x) \Vert \leq L_f \Vert x \Vert + \Vert \nabla f(0) \Vert
\leq L_f R + \Vert \nabla f(0) \Vert.
\]
This implies that $f$ is $(L_f R + \Vert \nabla f(0) \Vert)$-Lipschitz on $\overline{B}(0,R)$, therefore for any $x \in \overline{B}(0,R)$ one has
\[
f(x) - f(0) \leq (L_f R + \Vert \nabla f(0) \Vert) \Vert x - 0 \Vert,
\]
so that
\[
f(x) \leq (L_f R + \Vert \nabla f(0) \Vert) \Vert x \Vert + f(0)
\leq (L_f R + \Vert \nabla f(0) \Vert) R + f(0).
\]
\end{proof}

\section{Prototype Problem}\label{app:prototype}

First we prove that the LG and PL conditions of the integrand are inherited by the loss functional.

\begin{proof}[Proof of Lemma~\ref{lem:cond-iota}]
For any $h : X \to \R^l$, let $\nabla \iota_h : X \to \R^l$ be defined as $\nabla \iota_h(x) = \nabla_z \iota(x,h(x))$. Then by Proposition~\ref{lg_corollary_2}, $\int \Vert \nabla \iota_h(x)\Vert^2 d\mu(x) \leq \int 2L_\Ell(\iota(x,h(x)) - \iota(x,\cdot)_*) d\mu(x)$, which is finite if $h \in L^2(\mu,\R^l)$ by assumption. Now let $\nabla \iota_f \in L^2(\mu,\R^l)$ be defined analogously for $f$. For any $g \in L^2(\mu,\R^l)$ with $f \neq g$ one has
\[
\frac{\left\vert \Ell_\mu(g) - \Ell_\mu(f) - \langle \nabla \iota_f, g - f \rangle \right\vert}{\Vert g - f \Vert}
= \frac{\left\vert \int \iota(x,g(x)) - \iota(x,f(x)) - \langle \nabla_z \iota(x,f(x)), g(x) - f(x) \rangle d\mu(x) \right\vert}{\Vert g - f \Vert}
\]
\[
\leq \frac{\int \left\vert \iota(x,g(x)) - \iota(x,f(x)) - \langle \nabla_z \iota(x,f(x)), g(x) - f(x) \rangle \right\vert d\mu(x)}{\Vert g - f \Vert}
\]
\[
\leq \frac{\int \frac{L}{2} \Vert g(x) - f(x) \Vert^2 d\mu(x)}{\Vert g - f \Vert}
= \frac{L_\iota}{2} \frac{\Vert g - f \Vert^2}{\Vert g - f \Vert}
= \frac{L_\iota}{2} \Vert g - f \Vert.
\]
Thus
\[
\lim_{g \to f, f \neq g} \frac{\left\vert \Ell_\mu(g) - \Ell_\mu(f) - \langle \nabla \iota_f, g - f \rangle \right\vert}{\Vert g - f \Vert} = 0,
\]
implying that $\nabla \Ell_\mu(f) = \nabla \iota_f$.

In order to prove that $\Ell_\mu$ is $L_\iota$-LG note that
\begin{equation*}
\Vert \nabla \Ell_\mu(f) - \nabla \Ell_\mu(g) \Vert^2
= \int \Vert \nabla_z \iota(x,f(x)) - \nabla_z \iota(x,g(x)) \Vert^2 d\mu(x) \\
\leq \int L^2 \Vert f(x) - g(x) \Vert^2 d\mu(x) = L^2 \Vert f - g \Vert^2.
\end{equation*}

To prove the last part of the lemma, by \citet[Theorem~3A]{Rockafellar1976} we have that
\[
{\Ell_\mu}_* = \inf_{f \in L^2(\mu,\R^l)} {\Ell_\mu}(f) = \int \inf_{z \in \R^l} \iota(x,z) d\mu(x) = \int \iota(x,\cdot)_* d\mu(x) = \iota_*.
\]
Hence
\[
\frac{1}{2} \Vert \nabla {\Ell_\mu}(f) \Vert^2 
= \frac{1}{2} \int \Vert \nabla_z \iota(x,f(x)) \Vert^2 d\mu(x)
\geq \int \lambda (\iota(x, f(x)) - \iota(x,\cdot)_*) d\mu(x) = \lambda (\Ell_\mu(f) - {\Ell_\mu}_*),
\]
showing that $\Ell_\mu$ is $\lambda$-PL.
\end{proof}

We now prove Lemma~\ref{lem:inheritance-gj-lj-main} about the inheritance of the BJ and LJ properties from $N$ to $N_\mu$.

\begin{proof}[Proof of Lemma~\ref{lem:inheritance-gj-lj-main}]
First of all, let us prove that $\partial N_\mu(\theta) \eta$ is $\mu$-a.e. equal to $\partial_\theta N(\cdot,\theta)\eta$. Note that for $\mu$ almost every $x\in X$ we have that by Proposition~\ref{fundamental_theorem_of_calculus} the fact that $N(x,\cdot)$ is $\hat{K}_N(x)$-BJ implies that $\Vert N(x,\theta_1)-N(x,\theta_2) \Vert \leq \hat{K}_N(x) \Vert \theta_1 - \theta_2 \Vert$ if $[\theta_1, \theta_2] \subset D$. In order to prove that $\partial N_\mu(\theta) \eta$ is $\mu$-a.e. equal to $\partial_\theta N(\cdot,\theta)$ it suffices to prove that
\[
\lim_{\delta \to 0}\frac{\int \| N(x,\theta+\delta)-N(x,\theta)-\partial_\theta N(x,\theta) \delta\|^2 \; d\mu(x)}{\|\delta\|^2}=0
\]
for $\theta \in D$. Since $D$ is open, eventually $[\theta,\theta + \delta] \subset D$, so that $\|N(x,\theta+\delta)-N(x,\theta)\|^2\le \hat{K}_N(x)^2 \| \delta\|^2$. Hence, $\| N(x,\theta+\delta)-N(x,\theta)-\partial_\theta N(x,\theta)\delta\|^2\le (2 \hat{K}_N(x) \| \delta\|)^2$. As $\hat{K}_N^2$ is integrable with respect to $\mu$ by dominated convergence we get that 
\[
\lim_{\delta\to 0}\int\frac{ \| N(x,\theta+\delta)-N(x,\theta)-\partial_\theta N(x,\theta) \delta\|^2 }{\|\delta\|^2}\; d\mu(x) = \int\lim_{\delta\to 0}\frac{ \| N(x,\theta+\delta)-N(x,\theta)-\partial_\theta N(x,\theta) \delta\|^2 }{\|\delta\|^2}\; d\mu(x)=0.
\]

For any $f \in L^2(\mu,\R^l)$ note that
\begin{equation*}
\langle \partial N_\mu(\theta) \eta, f \rangle
= \int \langle \partial_\theta N(x,\theta) \eta, f(x) \rangle d\mu 
= \int \langle \eta, {\partial_\theta N(x,\theta)}^* f(x) \rangle d\mu
= \langle \eta, \int {\partial_\theta N(x,\theta)}^* f(x) d\mu \rangle,
\end{equation*}
where in the last equality we have used the linearity and continuity of the inner product. Hence the adjoint Jacobian ${\partial N_\mu(\theta)}^* \in \Ell(L^2(\mu,\R^l),\Theta)$ is given by
\[
{\partial N_\mu(\theta)}^* f = \int {\partial_\theta N(x,\theta)}^* f(x) d\mu(x).
\]

The rest of the lemma follows using \eqref{eq:jacobian-of-mu}. First, one has $\Vert \partial_\theta N(x,\theta) \Vert \leq \hat{K}_N(x)$ a.e. for all $\theta \in D$, so that
\[
\Vert \partial N_\mu(\theta) \eta \Vert
= \sqrt{\int \Vert \partial_\theta N(x,\theta) \eta \Vert^2 d\mu}
\leq \sqrt{\int \hat{K}_N(x)^2 \Vert \eta \Vert^2 d\mu(x)}
= \Vert \eta \Vert K_N
\]
for all $\theta \in D$, implying that $N_\mu$ is $K_N$-BJ on $D$. Second, one has $\Vert \partial_{\theta_1} N(x,\theta_1) - \partial_{\theta_2} N(x,\theta_2) \Vert \hat{L}_N(x) \Vert \theta_1 - \theta_2 \Vert$ a.e. for all $\theta_1, \theta_2 \in D$, so that
\[
\Vert (\partial N_\mu(\theta_1) - \partial N_\mu(\theta_2)) \eta \Vert 
= \sqrt{\int \Vert (\partial_{\theta_1} N(x,\theta_1) - \partial_{\theta_2} N(x,\theta_2)) \eta \Vert^2 d\mu(x)}
\]
\[
\leq \sqrt{\int \hat{L}_N(x)^2 \Vert \theta_1 - \theta_2 \Vert^2 \Vert \eta \Vert^2 d\mu(x)} = \Vert \eta \Vert L_N \Vert \theta_1 - \theta_2 \Vert
\]
for all $\theta \in D$, implying that $N_\mu$ is $L_N$-LJ on $D$.
\end{proof}

Now we prove Lemma~\ref{lemma_ntk_block_matrix} about the block matrix representation of the NTK.

\begin{proof}[Proof of Lemma~\ref{lemma_ntk_block_matrix}]
Consider the linear map $T : \Ell(L^2(\mu,\R^l), \R^{dl})$ defined as $Tf = \left[\frac{1}{\sqrt{d}} f(x_i) : i \in [1:d]\right] \in \R^{dl}$ for any $f \in L^2(\mu,\R^l)$, mapping a function to a block vector. One clearly has $\Vert f \Vert = \Vert Tf \Vert$, establishing the linear isometry $L^2(\mu,\R^l) \cong \R^{dl}$. Its adjoint $T^* \in \Ell(\R^{dl}, L^2(\mu,\R^l))$ is given by $T^*v(x_i) = \sqrt{d} v_i$ for any $v \in \R^{dl}$ and $i \in [1:d]$. Consider the representation $T \partial N_\mu(\theta) {\partial N_\mu(\theta)}^* T^* \in \Ell(R^{dl})$. One has for any $v \in \R^{dl}$ that
\[
\partial N_\mu(\theta) {\partial N_\mu(\theta)}^* T^* v (x_i)
= \frac{1}{d} \sum_{j=1}^d \partial_\theta N(x_i,\theta) {\partial_\theta N(x_j,\theta)}^* \sqrt{d} v_j
= \frac{1}{\sqrt{d}} \sum_{j=1}^d \partial_\theta N(x_i,\theta) {\partial_\theta N(x_j,\theta)}^* v_j,
\]
and therefore
\[
T \partial N_\mu(\theta) {\partial N_\mu(\theta)}^* T^* v
= \left[\frac{1}{d} \sum_{j=1}^d \partial_\theta N(x_i,\theta) {\partial_\theta N(x_j,\theta)}^* v_j : i \in [1:d]\right].
\]
Since for any block matrix $[A_{ij} : i,j \in [1:d]] \in \Ell(\R^{dl},\R^{dl})$ with blocks of size $l$ times $l$ one has $Av = \left[\sum_{j=1}^d A_{ij} v_j : i \in [1:d]\right]$, we have that the NTK $\partial N_\mu(\theta) {\partial N_\mu(\theta)}^*$ has the block matrix representation $\left[\frac{1}{d} \partial_\theta N(x_i,\theta) {\partial_\theta N(x_j,\theta)}^* : i,j \in [1:d]\right]$.
\end{proof}

Next, we detail two examples beyond supervised learning that can be translated to our prototype problem. The first is a popular unsupervised learning method, and shows that the framework being general enough to incorporate infinite data is a useful property. The second shows that even gradient regularization can be treated. The examples are intended to demonstrate that our framework covers real world learning problems. Their analysis is beyond the scope of this paper and is left for future work.

\begin{example}[Variational autoencoder]
The variational autoencoder (VAE) \citep{Kingmaetal2014, Rezendeetal2014} can be translated to our prototype problem as follows. First, denote the dataset by $\upsilon \in \Prob(Y)$, the prior distribution by $\zeta \in \Prob(Z)$ with $Z = \R^{l_Z}$, and the reparameterization distribution by $\omega \in \Prob(W)$ (which may or may not \citep{Jooetal2020} be equivalent to $\zeta$). The two components of the VAE are represented by the encoder map $E : Y \times \Theta_E \to \R^{l_E}$ differentiable in its second argument and the decoder map $D : Z \times \Theta_D \to \R^{l_D}$ differentiable in both arguments. A key component is the reparameterization function $r : W \times \R^{l_E} \to \R^{l_Z}$, which is measurable in its first argument with respect to $\omega$ and differentiable in its second argument, and satisfies the absolute continuity property $r(\cdot,z_E)_\# \omega \ll \zeta$ for all $ z_E \in \R^{l_E}$. These are combined into the map $N : X \times \Theta \to \R^l$ with $X = Y \times W$, $\Theta = \Theta_E \times \Theta_D$ and $l = l_E + l_D$, defined as $N(x,\theta) = N((y,w),(\theta_E,\theta_D)) = (E(y,\theta_E), D(r(w,E(y,\theta_E)),\theta_D))$ for $x = (y,w) \in X$ and $\theta = (\theta_E,\theta_D) \in \Theta$. Denoting $\mu = \upsilon \otimes \omega$ and assuming that $\int \Vert N(x,\theta) \Vert^2 d\mu(x)$ exists and is finite for all $ \theta \in \Theta$, we define the induced map $N_\mu : \Theta \to L^2(\mu,\R^l)$.

Let $\ell : Y \times \R^{l_D} \to \R$ be the function mapping a pair $(y,z_D)$ consisting of an input and a decoder output to minus the natural logarithm of the probability density function of the probability distribution parameterized by $z_D$, evaluated at $y$. Let $\beta > 0$ be a constant and $d_\zeta : \R^{l_E} \to \R$ the function mapping an encoder output $z_E$ to the Kullback-Leibler divergence $d_\zeta(z_E) = D_{KL}(r(\cdot,z_E)_\# \omega \Vert \zeta)$ of the posterior probability distribution $r(\cdot,z_E)_\# \omega$ parameterized by $z_E$ from the prior distribution $\zeta$. The integrand $\iota : X \times \R^l \to \R$ is then defined as $\iota((y,w),(z_E,z_D)) = \ell(y,z_D) + \beta d_\zeta(z_E)$, consisting of the reconstruction term and the divergence term, the latter being weighted by $\beta$ \citep{Higginsetal2017}. Training a VAE is exactly the minimization problem $\min_{(\theta_E,\theta_D) \in \Theta}\{ (\Ell_\mu \circ N_\mu)(\theta_E,\theta_D) \}$. One can analyze the behavior of this problem when optimized by gradient descent by expressing the Jacobian $\partial N(\theta_E,\theta_D) \in \Ell(\Theta, L^2(\mu,\R^l))$ in terms of the Jacobians $\partial E(\theta_E) \in \Ell(\Theta_E, L^2(\upsilon,\R^{l_E}))$ and $\partial D(\theta_D \in \Ell(\Theta_D, L^2(\zeta,\R^{l_D})))$, and by determining if the terms $\ell$, $\beta$ and $d_\zeta$ are such that the integrand $\iota$ satisfies the required conditions.
\end{example}

\begin{example}[Gradient regularized discriminator for generative adversarial networks]
Training generative adversarial networks is done via gradient descent-ascent, the analysis of which is beyond the scope of our paper. Nevertheless, we are going to present an example showing that the training of even gradient regularized discriminators can be translated to the prototype problem. Let $X=\R^k$ be the input space, $D : X \times \Theta \to \R$ be the discriminator mapping differentiable in both arguments, and $\rho, \gamma \in \Prob(X)$ be the real and generated data distributions. Let $\mu = \frac{1}{2}(\rho+\gamma) \in \Prob(X)$ be the mixture of the real and generated distributions. Define $N : X \times \Theta \to \R^l$ with $l=1+k$ as $N(x,\theta)=(D(x,\theta),\nabla_x D(x,\theta))$, so that the output space consists of the discriminator output and its gradient with respect to the input.

A general integrand is defined as $
\iota(x,z) = \frac{d\rho}{d\mu}(x) \iota_\rho(x,z) + \frac{d\gamma}{d\mu}(x)\iota_\gamma(x,z)$, with different integrands $\iota_\rho, \iota_\gamma : X \times \R^l \to \R$ corresponding to the real and generated distributions. An example is the Wasserstein GAN with gradient penalty \citep{Gulrajanietal2017}, corresponding to $\iota_\rho(x,(y,w)) = y - \beta (\Vert w \Vert - 1)^2$ and $\iota_\gamma(x,(y,w)) = -y - \beta (\Vert w \Vert - 1)^2$ with $z = (y,w)$ and some $\beta > 0$. Another is the original GAN \citep{Goodfellowetal2014} with $R_1$ regularization \citep{Meschederetal2018}, corresponding to $\iota_\rho(x,(y,w)) = \log(y) - \beta \Vert w \Vert^2$ and $\iota_\gamma(x,(y,w)) = \log(1-y)$ with some $\beta > 0$, which is used by \citet{Karrasetal2019}.
\end{example}

\section{Multilayer Perceptron}\label{app:mlp}

\subsection{Preliminaries}

Let $t:\R\to \R$ and $r:\R\to \R$ be functions that may also depend on other variables $a,b,\ldots$. We say that $t=O_{a,b,\ldots}(r)$ if there exists a constant $C=C(a,b,\ldots)>0$ such that $|t(y)|\le C|r(y)|$ for all $y\ge 0$. If we say $t=O_{a,b,\ldots}(r)$ for $y$ large enough we mean that there exists a constant $C'=C'(a,b,\ldots)\ge 0$ such that for $y\ge C'$ we have $|t(y)|\le C|r(y)|$. Note that typically these notions are the same, as if $t=O_{a,b,\ldots}(r)$ for $y$ large enough and $|t(y)/r(y)|$ is continuous for $y\ge 0$ then $|t(y)/r(y)|\le C''$ for $y\in[0,C']$ so $t=O_{a,b,\ldots}(r)$ where the implicit constant is $\max(C,C'')$. We say that $t=\Omega_{a,b,\ldots}(r)$ if $r = O_{a,b,\ldots}(t)$.

For any vector space $\R^n$ we always consider the usual Euclidean norm. That is, if $v=(v_1,\ldots,v_n)\in \R^n$ then
\[
\|v\|:=\sqrt{v_1^2+\cdots+v_n^2}.
\]
Given a function $g:\R\to\R$ we will abuse the notation and write $g(v)$ for the coordinate-wise application of $g$, i.e., $g(v):=(g(v_1),\ldots,g(v_n))$.

In the space of real matrices $\R^{n\times m}$ of size $n\times m$ there are several norms that we can consider. The operator norm of $A\in \R^{n\times m}$ is the one that we can consider the default one, and is given by the formula:
\[
\|A\|:= \sup_{\|v\|\le 1} \|Av\|.
\]
This norm enjoys some nice properties that we will use in the sequel. For any $A\in \R^{n\times m}$ and $B\in \R^{m\times s}$ we have that:
\begin{enumerate}
    \item\label{it:1} $\|A\|=\sqrt{\lambda_{\max}(AA^*)}=\sqrt{\lambda_{\max}(A^*A)}=\sigma_{\max}(A)$ where $\lambda_{\max}$ represents the largest eigenvalue of a square matrix, $A^*$ the adjoint of $A$ and $\sigma_{\max}$ the largest singular value of a matrix.
    \item \label{it:2} $\|AB\|\le \|A\|\|B\|$.
\end{enumerate}

On the other hand, this norm has the disadvantage that it does not come from an inner product defined in the space of matrices. Hence, we will sometimes need to consider also the Frobenius norm of a matrix $A\in \R^{n\times m}$ defined as:
\[
\|A\|_F:= \sqrt{\sum_{i=1}^n\sum_{j=1}^m |A_{i,j}|^2}.
\]
The important relation between these norms that we will use is the following fact:
\begin{enumerate}\setcounter{enumi}{2}    \item\label{it:3} For any $A\in \R^{n\times m}$, $\|A\|\le \|A\|_F$.
    \item\label{it:4} A vector $v\in \R^n$ can be seen as a matrix operator either in $\R^{n\times 1}$ or in $\R^{1\times n}ˇ$. In both cases, the operator norm of $v$ equals its Frobenius norm which moreover equals its usual Euclidean norm as a vector. Hence, we can talk about $\|v\|$ and assume that it refers to any of those definitions.

\end{enumerate}

In the sequel, we will be interested in giving a norm to a space that is a product combination of matrix spaces and regular Euclidean spaces. For instance, the parameter space as described in the introduction will be $\Theta=\prod_{i=1}^J \R^{m_{i-1} \times m_i}\times \R^{m_i}$ for some positive integers $m_i\ge 1$. Hence a parameter will be $\theta=(A_1,b_1,\ldots,A_J,b_J)$ where $A_i$ are matrices and $b_i$ are vectors. In this case we define
\[
\|\theta\|:=\sqrt{\|A_1\|_F^2+\|b_1\|^2+\cdots+\|A_J\|_F^2+\|b_J\|^2}.
\]

Normally distributed random variables will play a key role in our analysis. We denote a normally distributed random variable of mean $\mu$ and variance $\sigma^2$ by $\mathcal{N}(\mu,\sigma^2)$. If a random variable $X$ has distribution $\mathcal{N}(\mu,\sigma^2)$ recall that $aX+b$ has distribution $\mathcal{N}(a\mu+b,a^2\sigma^2)$ for any $a,b\in \R$.

A random vector or matrix is just one such that its entries are randomly initialized according to some distribution. In our case, we will mainly consider random vectors and matrices such that their entries are initialized with independent standard normal (i.e., $\mathcal{N}(0,1)$) random variables. By \citet[Theorem~3.1.1]{Vershynin2018} we have the following result:

\begin{lemma}\label{lem:ver-1}
There exists an absolute constant $C_1>0$ such that the following holds. Let $v\in \R^m$ be a random vector such that $b_i$ are independent standard normal distributed for all $1\le i\le m$. Then for any $t>0$ we have that
\[
\vert \Vert v \Vert - \sqrt{m} \vert
\leq C_1 t
\]
with probability at least $1 - 2e^{-t^2}$.
\end{lemma}

And for random matrices we have by \citet[Theorem~4.4.5]{Vershynin2018}:

\begin{lemma}\label{lem:ver-2}
There exists an absolute constant $C_2>0$ such that the following holds. Let $A\in \R^{n\times m}$ be a random matrix where all its entries are chosen independently with a standard normal distribution. Then for any $t>0$,
\[
\sigma_{\max}(A)
= \Vert A \Vert
\leq C_2(\sqrt{n}+\sqrt{m}+t)
\]
with probability at least $1 - 2e^{-t^2}$.
\end{lemma}

\subsection{Jacobian of the Neural Network Mapping, BJ and LG Properties}

Recall from Subsection~\ref{mlp} the definition of our MLP. It will be useful to introduce the following notation:

\begin{definition}[Right multiplier operator]\label{def:right-mult}
Let $n,m\ge 1$ be integers. For any $x\in \R^n$ and any $A\in \R^{m\times n}$ we define $M_x \in \Ell(\R^{m\times n}, \R^m)$ as the operator such that $M_x(A):=Ax$.
\end{definition}

Note that $\|M_x\|=\|x\|$ (its operator norm) for any $x$ where in $\R^{m\times n}$ we are choosing either the operator norm of the matrix or the Frobenius norm (i.e., both $\sup_{\Vert A \Vert \leq 1} \Vert Ax \Vert = \Vert x \Vert$ and $\sup_{\Vert A \Vert_F \leq 1} \Vert Ax \Vert = \Vert x \Vert$ hold).

\begin{definition}[Diagonal operator]\label{def:diag-op}
Let $n\ge 1$. For any $x\in \R^n$ we define $D_x\in \R^{n\times n}$ as the square matrix that has 0 everywhere except for the diagonal, where it equals $x$.
\end{definition}

In this case, it is easy to see that $\|D_x\|=\|x\|_\infty=\max_{1\le i\le n}\{|x_i|\}$ and $\|D_x\|_F=\|x\|$.

We will describe the Jacobian of the neural network mapping inductively as follows. First note that if we only had 1 layer the neural network mapping will be $N(x,\theta=(A,b))=Ax+b$. In this case, as this map is linear in $\theta$, its Jacobian is itself, meaning that if $\eta=(A',b')$ then $\partial_\theta N(x,\theta)\eta = A'x+b'$. For convenience, we will say that
\[
\partial_\theta N(x,\theta) = \left[ \begin{array}{cc} M_x & \Id \end{array} \right].
\]
This way, computing $\partial_\theta N(x,\theta)\eta$ can be regarded as $\left[ \begin{array}{cc} M_x & \Id \end{array} \right] \left[ \begin{array}{c} A' \\ b' \end{array} \right]=M_xA'+b'=A'x+b'$.

The usefulness of this notation comes when we add more layers. Indeed, if we have a 2-layer MLP ($J=2$) we can do the following. In this case, if $\theta=(A_1,b_2,A_2,b_2)$ we will say that $\theta_1=(A_1,b_1)$ and $\theta_2=(A_2,b_2)$. Then
\[
\partial_\theta N_2(x,\theta) = \left[ \begin{array}{ccc} A_2 D_{\phi'\left(\frac{1}{\sqrt{m}} N_{1}(x,\theta_{1})\right)} \frac{1}{\sqrt{m}} \partial_{\theta_{1}} N_{1}(x,\theta_{1 }) & M_{\phi\left(\frac{1}{\sqrt{m}} N_{1}(x,\theta_{1})\right)} & \Id \end{array} \right].
\]
Using the fact that $\partial_{\theta_1} N_1(x,\theta_1)= \left[ \begin{array}{cc} M_x & \Id \end{array} \right]$ we can conclude that
\[
\partial_\theta N_2(x,\theta) = \left[ \begin{array}{cccc} A_2 D_{\phi'\left(\frac{1}{\sqrt{m}} N_{1}(x,\theta_{1})\right)} \frac{1}{\sqrt{m}} M_x & A_2 D_{\phi'\left(\frac{1}{\sqrt{m}} N_{1}(x,\theta_{1})\right)} \frac{1}{\sqrt{m}} \Id & M_{\phi\left(\frac{1}{\sqrt{m}} N_{1}(x,\theta_{1})\right)} & \Id \end{array} \right].
\]
As before, if we want to evaluate this expression on a certain $\eta=(A_1',b_1',A_2',b_2')$ we would just have to multiply the previous block matrix by 
\[
\left[
\begin{array}{c} A_1' \\ b_1' \\ A_2' \\ b_2' \end{array} \right].
\]

Using this convention, the next proposition gives us a useful way of writing the Jacobian of the neural network mapping.

\begin{proposition}\label{prop:exp-jacobian}
Let $N(x,\theta)$ be a $J$-layer MLP as defined in Subsection~\ref{mlp}. Let $\theta=(\theta_1,\ldots,\theta_J)$ where $\theta_i=(A_i,b_i)\in \R^{m_{i-1} \times m_i}\times \R^{m_i}$ for $1\le i\le J$. Then
\[
\partial_{\theta_1} N_1(x,\theta_1) = \left[ \begin{array}{cc} M_x & \Id \end{array} \right]
\]
and for $2\le i\le J$,
\[
\partial_{\theta_{1:i}} N_i(x,\theta_{1:i}) = \left[ \begin{array}{ccc} A_i D_{\phi'\left(\frac{1}{\sqrt{m}} N_{i-1}(x,\theta_{1 : i-1})\right)} \frac{1}{\sqrt{m}} \partial_{\theta_{1 : i-1}} N_{i-1}(x,\theta_{1 : i-1}) & M_{\phi\left(\frac{1}{\sqrt{m}} N_{i-1}(x,\theta_{1 : i-1})\right)} & \Id \end{array} \right].
\]
\end{proposition}

\begin{proof}
It follows by applying the chain rule repeatedly on the expression of $N(x,\theta)$.
\end{proof}

First we will need a result about the concentration of the norms of the initial matrices and biases.

\begin{lemma}[Concentration of initial parameters]\label{lemma_initial_params}
Fix any constant $C>0$ and choose $\theta_0 = (A_{0,1},b_{0,1},\ldots,A_{0,J},b_{0,J})$ randomly and independently with distribution $\mathcal{N}(0,1)$ in each entry, except for those of $A_{0,J}$ with distribution $\mathcal{N}(0,\frac{1}{m})$. Then, with probability at least $1-4Je^{-\Omega_{\gamma_{1:J-1},k,l}(m)}$ we have that
\[
\|A_{0,1}\| =O_{k,\gamma_1}(\sqrt{m}),
\]
\[
\|A_{0,i}\| =O_{\gamma_{i-1},\gamma_i}(\sqrt{m})
\]
for all $2\le i \le J-1$,
\[
\|A_{0,J}\| =O_{\gamma_{J-1},l}(1),
\]
\[
\|b_{0,i}\| =O_{\gamma_i}(\sqrt{m})
\]
for all $1 \le i \le J-1$ and
\[
\|b_{0,J}\| =O_{l}(1).
\]
\end{lemma}
\begin{proof}
The result is an immediate consequence of Lemma~\ref{lem:ver-1} and Lemma~\ref{lem:ver-2}.
\end{proof}

First we have to prove the following:

\begin{lemma}[Neural network mapping is bounded]\label{lem:nn-map-bounded}
Let $N(x,\theta)$ be a $J$-layer MLP as defined in Subsection~\ref{mlp}. Fix any constant $C>0$ and choose $\theta_0 = (A_{0,1},b_{0,1},\ldots,A_{0,J},b_{0,J})$ randomly and independently with distribution $\mathcal{N}(0,1)$ in each entry, except for those of $A_{0,J}$ with distribution $\mathcal{N}(0,\frac{1}{m})$. Then, with probability at least $1-4Je^{-\Omega_{\gamma_{1:J-1},k,l}(m)}$ we have that for any  $\theta$ such that $\|\theta-\theta_0\| \le C\sqrt{m}$,
\[
\Vert N_i(x,\theta_{1:i}) \Vert \leq O_{C,\gamma_{1:i},k}(\sqrt{m}\sqrt{\Vert x \Vert^2+1})
\]
for $1 \leq i \leq J-1$ and
\[
\Vert N(x,\theta) \Vert \leq O_{C,\gamma_{1:J-1},k,l}(\sqrt{m}\sqrt{\Vert x \Vert^2+1}).
\]
\end{lemma}
\begin{proof}
Condition on the event of Lemma~\ref{lemma_initial_params} that happens with probability at least $1-4Je^{-\Omega_{\gamma_{1:J-1},k,l}(m)}$. By the triangle inequality and using property~\ref{it:3} we also have that if $\theta=(A_1,b_1,\ldots,A_J,b_J)$ then for all $1\le i\le J$,
\[
\|A_i\| =O_{C,\gamma_{1:i},k,l}(\sqrt{m}) \text{ and } \|b_i\| =O_{C,\gamma_{1:i},k,l}(\sqrt{m}).
\]

The proof of the result will then be by induction on $i$. The case $i=1$ is easy as
\[
\Vert N_1(x,\theta) \Vert = \Vert A_1 x + b_1 \Vert \leq \sqrt{\Vert A_1 \Vert^2 + \Vert b_1 \Vert^2} \sqrt{\Vert x \Vert^2 + 1}=O_{C,\gamma_{1:i},k,l}(\sqrt{m}\sqrt{\Vert x \Vert^2 + 1}).
\]
For $i>1$, we have that
\begin{align*}
\Vert N_i(x,\theta_{1:i}) \Vert & =  \left\Vert A_i \phi\left(\frac{1}{\sqrt{m}} N_{i-1}(x,\theta_{1:i-1})\right) + b_i \right\Vert \\
& \le  \left\Vert A_i \right\Vert \left\Vert\phi\left(\frac{1}{\sqrt{m}} N_{i-1}(x,\theta_{1:i-1})\right)\right\Vert + \left\Vert b_i \right\Vert.
\end{align*}
As $\phi(x)\leq |x|$ we have that
\begin{align*}
\Vert N_i(x,\theta_{1:i}) \Vert &  \le  \left\Vert A_i \right\Vert \left\Vert\frac{1}{\sqrt{m}} N_{i-1}(x,\theta_{1:i-1})\right\Vert + \left\Vert b_i \right\Vert\\
& = \frac{1}{\sqrt{m}}\left\Vert A_i \right\Vert \left\Vert N_{i-1}(x,\theta_{1:i-1})\right\Vert + \left\Vert b_i \right\Vert\\
& = O_{C,\gamma_{1:i},k,l}(\sqrt{m}\sqrt{\Vert x \Vert^2 + 1}).
\end{align*}\end{proof}

Our next goal is to prove that if we initialize randomly the weights of the neural network mapping, with probability tending to 1 as $m\to \infty$ we have that we will choose a \emph{good starting point}. That means that in a \emph{large} ball around the initial random parameter the BJ and LJ properties will be satisfied. More precisely:

\begin{theorem}[Neural network mapping is BJ]\label{thm:nn-map-bj}
Let $N(x,\theta)$ be a $J$-layer MLP as defined in Subsection~\ref{mlp}. Fix any constant $C>0$ and choose $\theta_0 = (A_{0,1},b_{0,1},\ldots,A_{0,J},b_{0,J})$ randomly and independently with distribution $\mathcal{N}(0,1)$ in each entry, except for those of $A_{0,J}$ with distribution $\mathcal{N}(0,\frac{1}{m})$. Then, with probability at least $1-4Je^{-\Omega_{\gamma_{1:J-1},k,l}(m)}$ we have that for any  $\theta$ such that $\|\theta-\theta_0\| \le C\sqrt{m}$,
\[
\Vert \partial_{\theta_{1:i}} N(x,\theta_{1:i}) \Vert \leq O_{C,\gamma_{1:i},k,\Vert \phi' \Vert_\infty}(\sqrt{\Vert x \Vert^2+1})
\]
for $1 \leq i \leq J-1$ and
\[
\Vert \partial_\theta N(x,\theta) \Vert \leq O_{C,\gamma_{1:J-1},k,l,\Vert \phi' \Vert_\infty}(\sqrt{\Vert x \Vert^2+1}).
\]
\end{theorem}
\begin{proof}
Condition on the event of Lemma~\ref{lemma_initial_params} that happens with probability at least $1-4Je^{-\Omega_{\gamma_{1:J-1},k,l}(m)}$. By the triangle inequality and using property~\ref{it:3} we also have that if $\theta=(A_1,b_1,\ldots,A_J,b_J)$ then for all $1\le i\le J$,
\[
\|A_i\| =O_{C,\gamma_{1:i},k,l}(\sqrt{m}) \text{ and } \|b_i\| =O_{C,\gamma_{1:i},k,l}(\sqrt{m}).
\]
 
We will prove this again by induction on the number of hidden layers. For the base case $i=1$ we have that by Proposition~\ref{prop:exp-jacobian} and the line after Definition~\ref{def:right-mult},
\[
\|\partial_{\theta_1} N_1(x,\theta_1)\| \le \sqrt{\|M_x\|^2+1} = O(\sqrt{\|x\|^2+1}).
\]
For the inductive case, again using Proposition~\ref{prop:exp-jacobian} we have that
\[
\Vert \partial_\theta N_i(x,\theta) \Vert \leq \sqrt{ \frac{1}{m} \Vert A_i \Vert^2 \Vert \phi' \Vert_\infty^2 \Vert \partial_{\theta_{1 : i-1}} N_{i-1}(x,\theta_{1 : i-1}) \Vert^2 + \frac{1}{m} \Vert N_{i-1}(x,\theta_{1 : i-1}) \Vert^2 + 1 }.
\]
By Lemma~\ref{lem:nn-map-bounded} and the induction hypothesis we conclude the result.
\end{proof}

We will need an extra auxiliary result in this section.

\begin{lemma}\label{lem:nn-map-lipschitz}
Let $N(x,\theta)$ be a $J$-layer MLP as defined in Subsection~\ref{mlp}. Fix any constant $C>0$ and choose $\theta_0 = (A_{0,1},b_{0,1},\ldots,A_{0,J},b_{0,J})$ randomly and independently with distribution $\mathcal{N}(0,1)$ in each entry, except for those of $A_{0,J}$ with distribution $\mathcal{N}(0,\frac{1}{m})$. Then, with probability at least $1-4Je^{-\Omega_{\gamma_{1:J-1},k,l}(m)}$ we have that for any  $\theta,\theta'$ such that $\max(\|\theta-\theta_0\|,\|\theta'-\theta_0\|) \le C\sqrt{m}$,
\[
\Vert N_i(x,\theta_{1:i})-N_i(x,\theta_{1:i}') \Vert \leq O_{C,\gamma_{1:i},k,\Vert \phi' \Vert_\infty}\left(\Vert \theta_{1:i} - \theta_{1:i}'\Vert\sqrt{\Vert x \Vert^2+1}\right)
\]
for $1 \leq i \leq J-1$ and
\[
\Vert N(x,\theta)-N(x,\theta') \Vert \leq O_{C,\gamma_{1:J-1}, k,l,\Vert \phi' \Vert_\infty}\left(\Vert \theta - \theta'\Vert\sqrt{\Vert x \Vert^2+1}\right).
\]
\end{lemma}
\begin{proof}
Condition on the event of Lemma~\ref{lemma_initial_params} that happens with probability at least $1-4Je^{-\Omega_{\gamma_{1:J-1},k,l}(m)}$. By the triangle inequality and using property~\ref{it:3} we also have that if $\theta=(A_1,b_1,\ldots,A_J,b_J)$ (resp. $\theta'$) then for all $1\le i\le J$,
\[
\|A_i\| =O_{C,\gamma_{1:i},k,l}(\sqrt{m}) \text{ and } \|b_i\| =O_{C,\gamma_{1:i},k,l}(\sqrt{m}) \text{ (resp. }A_i',b_i'\text{)}.
\]

We prove this again by induction on $i$. For $i=1$ we have that
\[
\Vert N_1(x,\theta_1)-N_1(x,\theta_1') \Vert = \Vert A_1 x + b_1 - A_1' x - b_1' \Vert \leq \Vert \theta_1 - \theta_1' \Vert \sqrt{\Vert x \Vert^2 + 1} = O(\Vert \theta_1 - \theta_1' \Vert\sqrt{\Vert x \Vert^2+1}).
\]
For larger $i$, we have that
\[
\Vert N_i(x,\theta_{1:i})-N_i(x,\theta_{1:i}') \Vert = \left\Vert A_i \phi\left(\frac{1}{\sqrt{m}} N_{i-1}(x,\theta_{1:i-1})\right) + b_i - A_i' \phi\left(\frac{1}{\sqrt{m}} N_{i-1}(x,\theta_{1:i-1}')\right) - b_i' \right\Vert
\]
\[
= \left\Vert (A_i-A_i') \phi\left(\frac{1}{\sqrt{m}} N_{i-1}(x,\theta_{1:i-1})\right) + A_i' \left(\phi\left(\frac{1}{\sqrt{m}} N_{i-1}(x,\theta_{1:i-1})\right) - \phi\left(\frac{1}{\sqrt{m}} N_{i-1}(x,\theta_{1:i-1}')\right)\right) + (b_i - b_i') \right\Vert
\]
\[
\leq \Vert A_i-A_i' \Vert \frac{1}{\sqrt{m}} \Vert N_{i-1}(x,\theta_{1:i-1}) \Vert + \Vert A_i' \Vert \Vert \phi' \Vert_\infty \frac{1}{\sqrt{m}} \Vert N_{i-1}(x,\theta_{1:i-1})-N_{i-1}(x,\theta_{1:i-1}') \Vert + \Vert b_i - b_i' \Vert
\]
\[
= O_{C,\gamma_{1:i},k,\Vert \phi' \Vert_\infty}\left(\Vert \theta_{1:i}-\theta_{1:i}' \Vert\sqrt{\Vert x \Vert^2+1} \right)
\]
where in the first inequality we have used induction for $i \in [1:J-1]$ and Lemma~\ref{lem:nn-map-bounded}. Therefore
\[
\Vert N(x,\theta)-N(x,\theta') \Vert \leq O_{C,\gamma_{1:J-1},k,l,\Vert \phi' \Vert_\infty}\left(\Vert \theta-\theta' \Vert\sqrt{\Vert x \Vert^2+1} \right).
\]
\end{proof}

We can now prove the last result of this subsection. Namely, that with probability tending to 1 as $m\to\infty$ we have that in a large ball around the initial point we have the LJ condition.

\begin{theorem}[Neural network mapping is LJ]\label{thm:nn-map-lj}
Let $N(x,\theta)$ be a $J$-layer MLP as defined in Subsection~\ref{mlp}. Fix any constant $C>0$ and choose $\theta_0 = (A_{0,1},b_{0,1},\ldots,A_{0,J},b_{0,J})$ randomly and independently with distribution $\mathcal{N}(0,1)$ in each entry, except for those of $A_{0,J}$ with distribution $\mathcal{N}(0,\frac{1}{m})$. Then, with probability at least $1-4Je^{-\Omega_{\gamma_{1:J-1},k,l}(m)}$ we have that for any  $\theta,\theta'$ such that $\max(\|\theta-\theta_0\|,\|\theta'-\theta_0\|) \le C\sqrt{m}$,
\[
\Vert \partial_\theta N(x,\theta)-\partial_\theta N(x,\theta') \Vert \leq O_{C,\gamma_{1:J-1}, k,l,\Vert \phi' \Vert_\infty,\Vert \phi' \Vert_L}\left(\frac{1}{\sqrt{m}}\Vert \theta - \theta'\Vert\sqrt{\Vert x \Vert^2+1}^{J-1}\right).
\]
\end{theorem}
\begin{proof}
Condition on the event of Lemma~\ref{lemma_initial_params} that happens with probability at least $1-4Je^{-\Omega_{\gamma_{1:J-1},k,l}(m)}$. By the triangle inequality and using property~\ref{it:3} we also have that if $\theta=(A_1,b_1,\ldots,A_J,b_J)$ (resp. $\theta'=(A_1',b_1',\ldots,A_J',b_J')$) then for all $1\le i\le J$,
\[
\|A_i\| =O_{C,\gamma_{1:i},k,l}(\sqrt{m}) \text{ and } \|b_i\| =O_{C,\gamma_{1:i},k,l}(\sqrt{m}) \text{ (resp. }A_i',b_i'\text{)}.
\]

The case $i=1$ of this result is trivial as in this case $\partial_{\theta} N_1(x,\theta)=\partial_{\theta'} N_1(x,\theta')$ by Proposition~\ref{prop:exp-jacobian}.

For the inductive case, again using Proposition~\ref{prop:exp-jacobian} we have that $\partial_{\theta_{1:i}} N_i(x,\theta_{1:i})-\partial_{\theta_{1:i}} N_i(x,\theta_{1:i}') $ equals (the following matrix is a row matrix, hence the adjoint that appears on the top left corner)
\[
\left[ \begin{array}{ccc} A_i D_{\phi'\left(\frac{1}{\sqrt{m}} N_{i-1}(x,\theta_{1 : i-1})\right)} \frac{1}{\sqrt{m}} \partial_{\theta_{1 : i-1}} N_{i-1}(x,\theta_{1 : i-1}) - A_i' D_{\phi'\left(\frac{1}{\sqrt{m}} N_{i-1}(x,\theta_{1 : i-1}')\right)} \frac{1}{\sqrt{m}} \partial_{\theta_{1 : i-1}'} N_{i-1}(x,\theta_{1 : i-1}') & \\ M_{\phi\left(\frac{1}{\sqrt{m}} N_{i-1}(x,\theta_{1 : i-1})\right)} - M_{\phi\left(\frac{1}{\sqrt{m}} N_{i-1}(x,\theta_{1 : i-1}')\right)} & \\ \Id - \Id \end{array} \right]^*.
\]
The idea now is to bound each of those terms one by one. Clearly the last one is just 0 so we can ignore it. We write the first one as
\[ \begin{array}{c} (A_i - A_i') D_{\phi'\left(\frac{1}{\sqrt{m}} N_{i-1}(x,\theta_{1 : i-1})\right)} \frac{1}{\sqrt{m}} \partial_{\theta_{1 : i-1}} N_{i-1}(x,\theta_{1 : i-1}) \\ + A_i' D_{\phi'\left(\frac{1}{\sqrt{m}} N_{i-1}(x,\theta_{1 : i-1})\right) - \phi'\left(\frac{1}{\sqrt{m}} N_{i-1}(x,\theta_{1 : i-1}')\right)} \frac{1}{\sqrt{m}} \partial_{\theta_{1 : i-1}} N_{i-1}(x,\theta_{1 : i-1}) \\+ A_i' D_{\phi'\left(\frac{1}{\sqrt{m}} N_{i-1}(x,\theta_{1 : i-1}')\right)} \frac{1}{\sqrt{m}} (\partial_{\theta_{1 : i-1}} N_{i-1}(x,\theta_{1 : i-1}) - \partial_{\theta_{1 : i-1}'} N_{i-1}(x,\theta_{1 : i-1}'))  \end{array}.
\]
The operator norm of this expression can then be bounded using $\Vert \phi' \Vert_\infty$ and $\Vert \phi' \Vert_L$. Hence the operator norm of the previous expression is at most
\[
 \frac{1}{\sqrt{m}} \Vert A_i - A_i' \Vert \Vert \phi' \Vert_\infty \Vert \partial_{\theta_{1 : i-1}} N_{i-1}(x,\theta_{1 : i-1}) \Vert 
\]
\[
+ \frac{1}{m} \Vert A_i' \Vert \Vert \phi' \Vert_L \Vert N_{i-1}(x,\theta_{1 : i-1}) - N_{i-1}(x,\theta_{1 : i-1}') \Vert \Vert \partial_{\theta_{1 : i-1}} N_{i-1}(x,\theta_{1 : i-1}) \Vert 
\]
\[
 + \frac{1}{\sqrt{m}} \Vert A_i' \Vert \Vert \phi' \Vert_\infty \Vert \partial_{\theta_{1 : i-1}} N_{i-1}(x,\theta_{1 : i-1}) - \partial_{\theta_{1 : i-1}'} N_{i-1}(x,\theta_{1 : i-1}') \Vert.
\]
Using Theorem~\ref{thm:nn-map-bj}, Lemma~\ref{lem:nn-map-bounded} and Lemma~\ref{lem:nn-map-lipschitz} the result then follows.

For the second term $M_{\phi\left(\frac{1}{\sqrt{m}} N_{i-1}(x,\theta_{1 : i-1})\right)} - M_{\phi\left(\frac{1}{\sqrt{m}} N_{i-1}(x,\theta_{1 : i-1}')\right)} = M_{\phi\left(\frac{1}{\sqrt{m}} N_{i-1}(x,\theta_{1 : i-1})\right) - \phi\left(\frac{1}{\sqrt{m}} N_{i-1}(x,\theta_{1 : i-1}')\right)}$, since $\Vert \phi' \Vert_\infty = \Vert \phi \Vert_L$, we can bound its operator norm by 
\[
\frac{1}{\sqrt{m}} \Vert \phi' \Vert_\infty \Vert N_{i-1}(x,\theta_{1 : i-1}) - N_{i-1}(x,\theta_{1 : i-1}') \Vert.
\]
By Lemma~\ref{lem:nn-map-lipschitz} and induction the result follows for this term as well. Putting this estimate together with we have that at level $i$ we have the estimate
\[
\|\partial_{\theta_{1:i}} N_i(x,\theta_{1:i})-\partial_{\theta_{1:i}} N_i(x,\theta_{1:i}') \|=O_{C,\gamma_{1:i}, k, \Vert \phi' \Vert_\infty, \Vert \phi' \Vert_L}\left(\frac{1}{\sqrt{m}}\Vert \theta_{1:i} - \theta_{1:i}'\Vert\sqrt{\Vert x \Vert^2+1}^{i-1}\right)
\]
In particular the result of the theorem follows:
\[
\Vert \partial_\theta N(x,\theta)-\partial_\theta N(x,\theta') \Vert \leq O_{C,\gamma_{1:J-1}, k,l,\Vert \phi' \Vert_\infty,\Vert \phi' \Vert_L}\left(\frac{1}{\sqrt{m}}\Vert \theta - \theta'\Vert\sqrt{\Vert x \Vert^2+1}^{J-1}\right).
\]
\end{proof}

Finally, we need to bound the norm of the initial output of the neural network.

\begin{lemma}[Neural network mapping is bounded]\label{lem:nn-map-bounded-initial}
Let $N(x,\theta)$ be a $J$-layer MLP as defined in Subsection~\ref{mlp}. Fix any constant $C>0$ and choose $\theta_0 = (A_{0,1},b_{0,1},\ldots,A_{0,J},b_{0,J})$ randomly and independently with distribution $\mathcal{N}(0,1)$ in each entry, except for those of $A_{0,J}$ with distribution $\mathcal{N}(0,\frac{1}{m})$. Then, with probability at least $1-4Je^{-\Omega_{\gamma_{1:J-1},k,l}(m)}$ we have that
\[
\Vert N(x,\theta_0) \Vert \leq O_{\gamma_{1:J-1},k,l}(\sqrt{\Vert x \Vert^2+1}).
\]
\end{lemma}
\begin{proof}
Condition on the event of Lemma~\ref{lemma_initial_params} that happens with probability at least $1-4Je^{-\Omega_{\gamma_{1:J-1},k,l}(m)}$.

We claim that
\[
\Vert N_i(x,\theta_{0,1:i}) \Vert \leq O_{\gamma_{1:i},k}(\sqrt{m}\sqrt{\Vert x \Vert^2+1})
\]
for $1 \leq i \leq J-1$. The proof of this claim will be by induction on $i$. The case $i=1$ is easy as
\[
\Vert N_1(x,\theta_{0,1}) \Vert = \Vert A_{0,1} x + b_{0,1} \Vert \leq \sqrt{\Vert A_{0,1} \Vert^2 + \Vert b_{0,1} \Vert^2} \sqrt{\Vert x \Vert^2 + 1}=O_{\gamma_1,k}(\sqrt{m}\sqrt{\Vert x \Vert^2 + 1}).
\]
For $i>1$, we have that
\begin{align*}
\Vert N_i(x,\theta_{0,1:i}) \Vert & =  \left\Vert A_{0,i} \phi\left(\frac{1}{\sqrt{m}} N_{i-1}(x,\theta_{0,1:i-1})\right) + b_{0,i} \right\Vert \\
& \le  \left\Vert A_{0,i} \right\Vert \left\Vert\phi\left(\frac{1}{\sqrt{m}} N_{i-1}(x,\theta_{0,1:i-1})\right)\right\Vert + \left\Vert b_{0,i} \right\Vert.
\end{align*}
As $\phi(x)\leq |x|$ we have that
\begin{align*}
\Vert N_i(x,\theta_{0,1:i}) \Vert &  \le  \left\Vert A_{0,i} \right\Vert \left\Vert\frac{1}{\sqrt{m}} N_{i-1}(x,\theta_{0,1:i-1})\right\Vert + \left\Vert b_{0,i} \right\Vert\\
& = \frac{1}{\sqrt{m}}\left\Vert A_{0,i} \right\Vert \left\Vert N_{i-1}(x,\theta_{0,1:i-1})\right\Vert + \left\Vert b_{0,i} \right\Vert\\
& = O_{\gamma_{1:i},k}(\sqrt{m}\sqrt{\Vert x \Vert^2 + 1}),
\end{align*}
proving the claim.

The result then follows since
\begin{align*}
\Vert N(x,\theta_0) \Vert &  \le  \left\Vert A_{0,J} \right\Vert \left\Vert\frac{1}{\sqrt{m}} N_{J-1}(x,\theta_{0,1:J-1})\right\Vert + \left\Vert b_{0,J} \right\Vert\\
& = \frac{1}{\sqrt{m}}\left\Vert A_{0,J} \right\Vert \left\Vert N_{J-1}(x,\theta_{0,1:J-1})\right\Vert + \left\Vert b_{0,J} \right\Vert\\
& = O_{\gamma_{1:J-1},k,l}(\sqrt{\Vert x \Vert^2 + 1}).
\end{align*}
\end{proof}

\subsection{BJ and LJ Bounds for General $\mu$}
Assume that the moments of $\mu$ up to order $2(J-1)$ are finite with
\[
M_{\mu,i} = \int \Vert \cdot \Vert^i d\mu
\]
for $i \in [0,2(J-1)]$ (with $M_{\mu,0}=1$). 

\begin{theorem}\label{theorem_mlp}
Fix any $C>0$. Let $\theta_0$ be chosen randomly as described above and define $D = \{\theta\in \Theta : \|\theta-\theta_0\| \le C\sqrt{m} \} \subset \Theta$. 

Then, with probability at least $1 - 4Je^{-\Omega_{\gamma_{1:J-1},k,l}(m)}$, the induced mapping $N_\mu$ is $K_N$-BJ and $L_N$-LJ on $D$ with
\[
K_N = O\left(\sqrt{M_{\mu,2}+1}\right),
\]
\[
L_N = O\left(\frac{1}{\sqrt{m}} \sqrt{\sum_{i=0}^{J-1}\binom{J-1}{i}M_{\mu,2(J-1-i)}}\right),
\]
and
\[
\Vert N_\mu(\theta_0) \Vert = O\left( \sqrt{M_{\mu,2}+1} \right).
\]
Where these last three implicit constants depend on $C,\gamma_{1:J-1},k,l,\|\phi'\|_\infty,\Vert \phi' \Vert_L$.
\end{theorem}
\begin{proof}
By Theorem~\ref{thm:nn-map-bj} and Theorem~\ref{thm:nn-map-lj} we have that $N(x,\cdot)$ is $O(\sqrt{\Vert x \Vert^2 + 1})$-BJ and $O(\frac{1}{\sqrt{m}}\sqrt{\Vert x \Vert^2 + 1}^{J-1})$-LJ. These facts combined with Lemma~\ref{lem:inheritance-gj-lj-main} we have the first two claims of the result with probability at least $1 - 4Je^{-\Omega_{\gamma_{1:J-1},k,l}(m)}$. By Lemma~\ref{lem:nn-map-bounded-initial} we have an estimate of $\|N(x,\theta_0)\|$ for every $x$. Integrating over $\mu$ the result follows. As we conditioned on the same event happening with probability at least $1 - 4Je^{-\Omega_{\gamma_{1:J-1},k,l}(m)}$ the result follows.
\end{proof}

\subsection{Sampling the Dataset from a Data Generating Distribution}

Let $\nu\in \mathcal{P}(\R^k)$ satisfy the Lipschitz concentration property (see Section~\ref{mlp}). Suppose that $x\in \R^k$ is a random variable with distribution $\nu$. As the norm function $\|\cdot\|:\R^k\to \R$ is 1-Lipschitz we have that
\[
\nu\left(\left\{
x \in \R^k : \left\vert \Vert x \Vert - \int \Vert \cdot \Vert d\nu \right\vert > t
\right\}\right)
\leq 2e^{-ct^2},
\]
i.e., the norm distributed according to $\nu$ is a sub-Gaussian random variable. 

Let now $x_1,\ldots,x_d$ be $d$ i.i.d. (independent and identically distributed) random variables with distribution $\nu$. The empirical measure will be now $\mu:=\frac{1}{d}\sum_{i=1}^d \delta_{x_i}$. Recall from \citep{Vershynin2018} the notion of sub-Gaussian random variable and from \citep{Vladimirovaetal2020, Kuchibhotlaetal2022} that of sub-Weibull random variable. Note that the latter is parameterized by a constant which is denoted by $p$ on \citep{Vladimirovaetal2020} and by $\alpha$ in \citep{Kuchibhotlaetal2022} and the relation between these is $p\alpha=1$. We are going to use the former parameterization. Sub-Gaussian variables are sub-Weibull with parameter $p=\frac{1}{2}$, while sub-exponential variables are sub-Weibull with parameter $p=1$.

To prove Theorem~\ref{theorem_mlp_empirical}, we need the following lemmas:

\begin{lemma} \label{sub_weibull_x_norm}
Let $\nu\in \mathcal{P}(\R^k)$ be a distribution satisfying the Lipschitz concentration property. Suppose that $x$ is a random variable with distribution $\nu$. Then $\|x\|$ has sub-Gaussian distribution, $\|x\|^{2t}$ has sub-Weibull distribution with parameter $p=t$ and $a_1\|x\|^{2t_1} + \cdots + a_n\|x\|^{2t_n}$ for $a_1,\cdots,a_n \in \R$ has sub-Weibull distribution with parameter $p=\max\{t_1,\cdots,t_n\}$.
\end{lemma}
\begin{proof}
We have already proved the first claim of the lemma. The second and third follow by \citet[Proposition~2.3]{Vladimirovaetal2020}.
\end{proof}

Sub-Weibull random variables concentrate around their means in a similar way as sub-exponential variables do. More precisely we have the following version of Bernstein's inequality (generalizing \citet[Corollary~2.8.3]{Vershynin2018}, which corresponds to the case $p=1$):

\begin{lemma} \label{sub_weibull_bernstein}
Let $x_1,\cdots,x_n$ be i.i.d. mean zero sub-Weibull random variables with parameter $p \geq 1$. Then one has
\[
\left\vert \frac{1}{n} \sum_{i=1}^n x_i \right\vert \geq t
\]
with probability at most
\[
2e^{-c_p \min\left( \frac{n t^2}{K_p^2}, \left(\frac{n t}{K_p}\right)^{\frac{1}{p}} \right)}
\]
for some absolute constant $c_p>0$, any $t>0$ and with $K_p$ being the sub-Weibull norm with parameter $p$ of the $X_i$s.
\end{lemma}
\begin{proof}
By \citet[Theorem~3.1]{Kuchibhotlaetal2022}, one has
\[
\left\vert \frac{1}{n} \sum_{i=1}^n x_i \right\vert \geq 2eCK_p\left( \sqrt{\frac{\hat{t}}{n}} + \frac{4^p}{\sqrt{2}} \frac{\hat{t}^p}{n} \right)
\]
with probability at most $2e^{-\hat{t}}$ with a specific constant $C$ depending on $p$. There exists $T>0$ such that for $\hat{t} \leq T$ one has
\[
\sqrt{\frac{\hat{t}}{n}} \geq \frac{4^p}{\sqrt{2}} \frac{\hat{t}^p}{n}
\]
and for $\hat{t} \geq T$ one has
\[
\sqrt{\frac{\hat{t}}{n}} \leq \frac{4^p}{\sqrt{2}} \frac{\hat{t}^p}{n}.
\]
Let $\hat{t} \leq T$, so that
\[
\left\vert \frac{1}{n} \sum_{i=1}^n x_i \right\vert \geq 4eCK_p\sqrt{\frac{\hat{t}}{n}}
\]
with probability at most $2e^{-\hat{t}}$. Letting $t = 4eCK_p\sqrt{\frac{\hat{t}}{n}}$ leads to
\[
\left\vert \frac{1}{n} \sum_{i=1}^n x_i \right\vert \geq t
\]
with probability at most $
2e^{-\frac{1}{16e^2C^2}\frac{nt^2}{K_p^2}}$. Now let $\hat{t} \geq T$, so that
\[
\left\vert \frac{1}{n} \sum_{i=1}^n x_i \right\vert \geq 4eCK_p\frac{4^p}{\sqrt{2}} \frac{\hat{t}^p}{n}
\]
with probability at most $2e^{-\hat{t}}$. Letting $t = 4eCK_p\frac{4^p}{\sqrt{2}} \frac{\hat{t}^p}{n}$ leads to
\[
\left\vert \frac{1}{n} \sum_{i=1}^n x_i \right\vert \geq t
\]
with probability at most
\[
2e^{-\left(\frac{\sqrt{2}}{4eC4^p}\right)^{\frac{1}{p}} \left(\frac{n t}{K_p}\right)^{\frac{1}{p}}}.
\]
These two cases lead to the conclusion with \[
c_p = \max\left( \frac{1}{16e^2C^2}, \left(\frac{\sqrt{2}}{4eC4^p}\right)^{\frac{1}{p}} \right).
\]
\end{proof}

\begin{corollary}\label{cor:concentration-sub-weibull}
Let $\nu\in \mathcal{P}(\R^k)$ be a distribution satisfying the Lipschitz concentration property. Let $x_1,\ldots,x_d$ be i.i.d. random variables with distribution $\nu$. Let $p\ge 1$ be an integer. Then 
\[
\nu^{\otimes d}\left(\left\{
(x_1,\cdots,x_d) : \left\vert \frac{1}{d} \sum_{i=1}^d (\Vert x_i \Vert^2 + 1)^p - \sum_{j=0}^p \binom{p}{j} \int \Vert \cdot \Vert^{2(p-j)} d\nu \right\vert > t
\right\}\right)
\]
\[
\leq 2e^{-c_p \min\left( \frac{d t^2}{C_p^2}, \left(\frac{d t}{C_p}\right)^{\frac{1}{p}} \right)}
\]
with some absolute constant $c_p>0$ and any $t>0$ and $C_p$ depending on $\nu$ and $p$.
\end{corollary}
\begin{proof}
By Lemma~\ref{sub_weibull_x_norm}, $(\Vert x_i \Vert^2 + 1)^p$ is $p$-sub-Weibull. One also has the expectation
\[
\int (\Vert x_i \Vert^2 + 1)^p d\nu(x_i)
= \int \sum_{j=0}^p \binom{p}{j} \Vert x_i \Vert^{2(p-j)} d\nu(x_i)
= \sum_{j=0}^p \binom{p}{j} \int \Vert \cdot \Vert^{2(p-j)} d\nu.
\]
Now apply Lemma~\ref{sub_weibull_bernstein} to the zero-mean $p$-sub-Weibull variables $(\Vert x_i \Vert^2 + 1)^p - \sum_{j=0}^p \binom{p}{j} \int \Vert \cdot \Vert^{2(p-j)} d\nu$.
\end{proof}

\subsection{General Convergence Result for Empirical Measures Sampled from a Fixed Distribution}

The goal of this subsection is to prove one of the main results of the paper, Theorem~\ref{thm:general-conv-mlp}.

\begin{proof}[Proof of Theorem~\ref{thm:general-conv-mlp}]
Let $f(m)$ be as described in Section~\ref{mlp}, i.e., a function such that $f(m)\to 0$ and $\sqrt{m}f(m)\to \infty$ as $m\to \infty$. Let $D=B(\theta_0, \sqrt{m}f(m))$. By Theorem~\ref{theorem_mlp_empirical} we have that with probability at least 
\[
1 - 4Je^{-\Omega_{\gamma_{1:J-1},k,l}(m)}
-2e^{-c_1 d \min\left( \frac{\epsilon_K^2}{C_{\nu,1}^2}, \frac{\epsilon_K}{C_{\nu,1}} \right)} - 2e^{-c_J \min\left( \frac{d \epsilon_L^2}{C_{\nu,J}}, \left(\frac{d \epsilon_L}{C_{\nu,J}}\right)^{\frac{1}{J-1}} \right)},
\]
we have that $N_\mu$ is $K_N$-BJ and $L_N$-LJ on $D$ with $K_N = O_T(1)$, $L_N = O_T\left(\frac{1}{\sqrt{m}}\right)$ and $\Vert N_\mu(\theta_0) \Vert = O_T(1)$.

By the third fact and Lemma~\ref{lemma_initial_loss_value}, we have that
\[
\Ell_\mu(N_\mu(\theta_0)) \leq (L_\Ell \Vert N_\mu(\theta_0) \Vert + \Vert \nabla \Ell_\mu(0) \Vert) \Vert N_\mu(\theta_0) \Vert + \Ell_\mu(0).
\]
By \citet[Theorem~3A]{Rockafellar1976},
\[
{\Ell_\mu}_* = \int {\iota(x,\cdot)}_* d\mu(x) = \iota_*.
\]
Since $\iota(x,\cdot)$ is $\lambda_\Ell$-PL, one has
\[
\Ell_\mu(0) = \int \iota(x,0) d\mu(x) \leq \int \frac{1}{2\lambda_\Ell} \Vert \nabla_z \iota(x,0) \Vert^2 + \iota(x,\cdot)_* d\mu = \frac{1}{2\lambda_\Ell} \Vert \nabla \Ell_\mu(0) \Vert^2 + {\Ell_\mu}_*.
\]
As $\nabla_z \iota(\cdot,0) : \R^k \to \R^l$ is $L_\Ell'$-Lipschitz and $\|\cdot\|:\R^k\to \R$ is 1-Lipschitz we have that $\Vert \nabla_z \iota(\cdot,0) \Vert: \R^k \to \R$ is $L_\Ell'$-Lipschitz as well, so that by Lipschitz concentration we have that
\[
\nu\left(\left\{
x \in \R^k : \left\vert \Vert \nabla_z \iota(x,0) \Vert - \int \Vert \nabla_z \iota(\cdot,0) \Vert d\nu \right\vert > t
\right\}\right)
\leq 2e^{-\frac{c_\nu t^2}{L_\Ell'^2}},
\]
i.e., $\Vert \nabla_z \iota(x,0) \Vert$ distributed according to $\nu$ is a sub-Gaussian random variable. By \citet[Lemma~2.7.6]{Vershynin2018}, $\Vert \nabla_z \iota(x,0) \Vert^2$ is sub-exponential (or $1$-sub-Weibull), so that by \citet[Corollary~2.8.3]{Vershynin2018} (or Lemma~\ref{sub_weibull_bernstein}) we have (since $\int \Vert \nabla_z \iota(\cdot,0) \Vert^2 d\nu = \Vert \nabla \Ell_\nu(0) \Vert^2 \in \R$) that
\[
\left\vert \Vert \nabla \Ell_\mu(0) \Vert^2 - \Vert \nabla \Ell_\nu(0) \Vert^2 \right\vert \leq \epsilon_\Ell
\]
with probability at least
\[
1-2e^{-c_1 d \min\left( \frac{\epsilon_\Ell^2}{C_{\nu,L_\Ell'}^2}, \frac{\epsilon_\Ell}{C_{\nu,L_\Ell'}} \right)}
\]
for some constants $c_1,C_{\nu,L_\Ell'}$. Condition on this event as well. We then have
\[
\Vert \nabla \Ell_\mu(0) \Vert \leq \sqrt{\int \Vert \nabla_z \iota(\cdot,0) \Vert^2 d\nu + \epsilon_\Ell} = O_T(1)
\]
and
\[
\Ell_\mu(0) \leq \frac{1}{2\lambda_\Ell} \left( \int \Vert \nabla_z \iota(\cdot,0) \Vert^2 d\nu + \epsilon_\Ell \right) + \iota_* = O_T(1).
\]
We then have 
\[
\Ell_\mu(N_\mu(\theta_0)) = O_T(1),
\]
so that (noting that $\Ell_\mu(N_\mu(\theta_0))-{\Ell_\mu}_* \geq 0$, so $K_\Ell \geq 0$ by definition)
\[
K_\Ell := \sqrt{2 L_\Ell(\Ell_\mu(N_\mu(\theta_0))-{\Ell_\mu}_*)}
= O_T(1).
\]

Also, by Assumption~\ref{assumption_lambda_min}, with probability at least $1-\epsilon$ we have that 
\[
\lambda_{\min}(\partial N_\mu(\theta_0) {\partial N_\mu(\theta_0)}^*) = \Omega(1)
\]
where the implicit constant may depend on $d,J,k,l,\gamma_{1:J-1},\phi$. Combining all these events we have that with probability at least
\[
1- \epsilon_\lambda
- 4Je^{-\Omega_{\gamma_{1:J-1},k,l}(m)} 
- 2e^{-c_1 d \min\left( \frac{\epsilon_K^2}{K_{\nu,1}^2}, \frac{\epsilon_K}{K_{\nu,1}} \right)}
- 2e^{-c_J \min\left( \frac{d \epsilon_L^2}{C_{\nu,J}}, \left(\frac{d \epsilon_L}{C_{\nu,J}}\right)^{\frac{1}{J-1}} \right)}
-2e^{-c_1 d \min\left( \frac{\epsilon_\Ell^2}{C_{\nu,L_\Ell'}^2}, \frac{\epsilon_\Ell}{C_{\nu,L_\Ell'}} \right)}
\]
all previous estimates hold.

By the above, we have that there exists absolute constants $C_K,C_L,C_\Ell$ and $C_\lambda$ such that for $m$ sufficiently large we have $K_N \leq C_K$, $L_N \leq \frac{1}{\sqrt{m}} C_L$, $K_\Ell \leq C_\Ell$ and $\lambda_{\min}(\partial N_\mu(\theta_0) {\partial N_\mu(\theta_0)}^*) \geq C_\lambda$.

Now we are going to exploit lazy training, i.e., the fact that $L_N$ decreases proportionally with $\frac{1}{\sqrt{m}}$. By Lemma~\ref{lemma_lazy_training}, we have that $N_\mu$ is $\lambda_N$-UC on $D$ with
\[
\lambda_N \geq C_\lambda - 2 C_K C_L f(m),
\]
which is positive for large enough $m$ since $\lim_{m\to\infty} f(m)=0$. In fact, we can assume that $\lambda_N>C_\lambda/2$ for $m$ large enough (depending on the variables in $T$).

Thus we have the following estimates:
\[
K = K_N K_\Ell =O_T(1), \quad L=K_N^2 L_\Ell + K_\Ell L_N =O_T(1), \text{ and } \lambda = \lambda_N \lambda_\Ell =\Omega_T(1).
\]
Letting $\alpha \in (0,\frac{2}{L})$ we define
\[
q= 1 + L\lambda\alpha^2 - 2\lambda\alpha = O_{\alpha,T}(1).
\]

Therefore if we let $R = \frac{\alpha K}{1-\sqrt{q}}$ note that we have $R=O_{\alpha,T}(1)$. For the hypotheses of Theorem~\ref{theorem_gd_f_circ_F} to hold, we need that $\overline{B}(\theta_0,R) \subset D$, for which it is sufficient to have $R < \sqrt{m}f(m)$. By the above bounds and the fact $\lim_{i\to\infty} \sqrt{m}f(m)=\infty$, this clearly holds for sufficiently large $m$, and the proof is complete.
\end{proof}

Now we prove that the Lipschitz constant of the trained MLP is bounded.

\begin{proof}[Proof of Lemma~\ref{lemma_lip_const}]
By the definition of $N(x,\theta)$, we have that
\[
\Vert N(\cdot,\theta) \Vert_L \leq \left(\frac{\Vert \phi \Vert_L}{\sqrt{m}}\right)^{J-1} \prod_{j=1}^J \Vert A_j \Vert.
\]
We also have that $\Vert \theta_* - \theta_0 \Vert = O_{T,\alpha}(1)$. By this, Lemma~\ref{lemma_initial_params}, the triangle inequality and using property~\ref{it:3}, we have (denoting $\theta_*=(A_{*,1},b_{*,1},\cdots,A_{*,J},b_{*,J})$) that
\[
\|A_{*,i}\| =O_{T,\alpha}(\sqrt{m})
\]
for all $1\le i \le J-1$ and
\[
\|A_{*,J}\| =O_{T,\alpha}(1),
\]
so that
\[
\Vert N(\cdot,\theta_*) \Vert_L \leq \left(\frac{\Vert \phi \Vert_L}{\sqrt{m}}\right)^{J-1} \prod_{j=1}^J \Vert A_{*,j} \Vert = O_{T,\alpha}(1).
\]
\end{proof}

\section{Experiments}\label{app:experiments}

The experiments were implemented in the JAX framework \citep{Jax2018}. Table~\ref{experiments_table} contains the hyperparameter choices for each experiment. The row labeled \# contains the number of samples (of $\theta_0$ and $\mu$) taken to compute the expectations and standard deviations. The tempered Gaussian-error linear unit (GELU) activation is 
\[
\phi(s) 
= \frac{1}{t} \frac{1}{2} ts \left(\erf\left(\frac{ts}{\sqrt{2}}\right) + 1\right)
= \frac{1}{2} s \left(\erf\left(\frac{ts}{\sqrt{2}}\right) + 1\right)
\]
with $\erf$ being the Gauss error function. We have used the temperature parameter $t=16$ in all experiments. As data normalization we divided the pixel values of each sample by $\sqrt{k}$. In the experiment about generalization in Subsection~\ref{experiment_generalization}, the data generating distribution was a random subset of MNIST of size $16384$.

\begin{table}[h]
\caption{Hyperparameters of experiments.}
\label{experiments_table}
\vskip 0.15in
\begin{center}
\begin{small}
\begin{sc}
\begin{tabular}{lccccc}
\toprule
Hyperparameter & \makecell{Learning Rate \\ Transfer} & \makecell{Concentration of \\ $\lambda_{\min}$ of the NTK \\ at Initialization} & Lazy Training & \makecell{Implicit \\ Regularization} & \makecell{Generalization \\ Error} \\
\midrule
$\alpha$		& N/A & N/A & $0.1$ & $0.1$ & $0.1$ \\
$d$				& $64$ & $16$ & $16$ & $64$ & N/A \\
$\lambda_\Ell$	& $10^{-4}$ & N/A & $10^{-4}$ & $10^{-4}$ & $10^{-4}$ \\
\#				& $100$ & $10000$ & $100$ & $100$ & $100$ \\
\bottomrule
\end{tabular}
\end{sc}
\end{small}
\end{center}
\vskip -0.1in
\end{table}

\end{document}